%% file: aistats_paper.tex
\documentclass[twoside]{article}
\usepackage[accepted]{aistats2025}

\usepackage{microtype}
\usepackage{graphicx}
\usepackage{subfigure}
\usepackage{booktabs} %
\usepackage{algorithm}
\usepackage{algorithmic}
\usepackage{comment}
\usepackage{colortbl}
\usepackage{graphicx}
\usepackage{multirow}
\usepackage{subfigure}
\usepackage{bbm}
\usepackage{color}

\usepackage{hyperref}

\usepackage{amsmath}
\usepackage{amssymb}
\usepackage{mathtools}
\usepackage{amsthm}
\allowdisplaybreaks

\usepackage[capitalize,noabbrev]{cleveref}

\theoremstyle{plain}
\newtheorem{theorem}{Theorem}[section]
\newtheorem{proposition}[theorem]{Proposition}
\newtheorem{lemma}[theorem]{Lemma}
\newtheorem{corollary}[theorem]{Corollary}
\theoremstyle{definition}

\newtheorem{assumption}[theorem]{Assumption}

\theoremstyle{remark}

\input{notation_song.tex}

\input{notation_danny.tex}

\input{notation_leyang.tex}
\input{notation_overwrites.tex}

\usepackage[round]{natbib}

\begin{document}
\runningauthor{Daniel J. Williams$^*$, Leyang Wang$^*$, Qizhen Ying, Song Liu, Mladen Kolar}

\runningtitle{High-Dimensional Differential Parameter Inference in Exponential Family using Time Score Matching}

\twocolumn[
\aistatstitle{High-Dimensional Differential Parameter Inference in 
Exponential Family using Time Score Matching}
\aistatsauthor{
    \href{mailto:daniel.williams@bristol.ac.uk}{Daniel J. Williams}$^{*,1}$ \And
    \href{mailto:leyang.wang@bristol.ac.uk}{Leyang Wang}$^{*,1}$ \And
    \href{mailto:exet6278@ox.ac.uk}{Qizhen Ying}$^{1}$}
\vspace{6pt}
\aistatsauthor{
    \href{mailto:song.liu@bristol.ac.uk}{Song Liu}$^{1}$ \And
    \href{mailto:mkolar@marshall.usc.edu}{Mladen Kolar}$^{2}$}
\vspace{6pt}
\aistatsaddress{ $^1$School of Mathematics, University of Bristol \And $^2$USC Marshall School of Business}]

\begin{abstract}
This paper addresses differential inference in time-varying parametric probabilistic models, like graphical models with changing structures. Instead of estimating a high-dimensional model at each time point and estimating changes later, we directly learn the differential parameter, i.e., the time derivative of the parameter. The main idea is treating the time score function of an exponential family model as a linear model of the differential parameter for direct estimation. We use time score matching to estimate parameter derivatives. We prove the consistency of a regularized score matching objective and demonstrate the finite-sample normality of a debiased estimator in high-dimensional settings. Our methodology effectively infers differential structures in high-dimensional graphical models, verified on simulated and real-world datasets. The code reproducing our experiments can be found at: \url{https://github.com/Leyangw/tsm}.
\end{abstract}

\section{Introduction}

In non-stationary environments, the data-generating process varies over time due to factors like news, geopolitical events, and economic reports \citep{Lu2019,quinonero2022dataset}. 
Understanding these changes is crucial for many applications.

When probabilistic model parameters change over time, learning their time-derivative, or differential parameter, can be beneficial. Although time-varying models are well-studied \citep{kolar2011,Kolar2012,gibberd2017regularized,yang2018estimating}, few focus on differential parameters. Learning these is advantageous: they reveal the underlying dynamics of systems, such as the SIR model \citep{tang2020review}. Differential parameters can also be more interpretable; stationary parameters become zero after differentiation, making the estimation target \emph{sparse}. This helps in handling high-dimensional overparametrized models 
\citep{hastie2015statistical,Wainwright_2019} and inferring with fewer samples.

Differential model estimation has been considered in a ``discrete setting'' \citep{zhao2014,liu2017,zhao2019direct,Byol2021}, where the focus is on identifying changes in the model between two discrete time points. However, in this paper, we learn the differential parameter in a continuous setting, where the parameter is a continuous function of the time.

We propose an efficient estimator for the differential parameter in high-dimensional probabilistic models. Our method estimates the differential parameter directly, without needing to estimate the parameters themselves. By addressing an $\ell_1$-regularized objective, our estimator achieves consistency in high-dimensional contexts, with a convergence rate dependent only on the dimensionality of the time-varying parameters. The debiased estimator shows asymptotic normality, making it suitable for parameter inference, and it efficiently estimates complex models without evaluating the normalizing term. We validate our theorems through synthetic experiments and demonstrate superior performance compared to a recent time-varying parameter estimation method \citep{yang2018estimating}. This is the first work to tackle the differential parameter estimation in exponential family in a continuous setting.

\section{Background}
Let $\mb{x} \in \mathbb{R}^d$ be a random sample generated from a distribution whose density function is $q(\mb{x})$. We also define a parametric model $\p$
\begin{align}
\label{eq:exp_model}
    \p = \frac{\pbar}{z(\mb\theta)},  \; \; z(\mb\theta) = \int \pbar d\mb{x}, 
\end{align}
where $\pbar$ is known analytically and $z(\mb\theta)$ is the normalizing constant which may be hard to compute or approximate. 

Before introducing the proposed approach, we first review two important notions, high dimensional probabilistic graphical model and score matching. 

\subsection{Estimating Probabilistic Model Parameters in High-dimensional Settings}

Suppose we are interested in using samples drawn from $q$ to learn an over-parametrized model $\p$ with more parameters than  necessary to describe the data. The ``true model" generating data might be simpler with few parameters. Sparse probabilistic model estimation aims to find a \emph{sparse parameterization} of the high-dimensional model that best fits the data. Previous studies have used sparsity-inducing regularization (e.g., $\ell_1$ penalty) along with the likelihood function on the dataset $\mathcal{D}$ \citep{Tibshirani1996,hastie2015statistical,yuan2007,drton2017structure}.
Consider a pairwise graphical model
\begin{align*}
    \bar{p}(\boldx; \boldTheta) := \exp\left( \sum_{i \le j} \Theta_{i,j} f(x_i, x_j) \right), i,j \in [d].
\end{align*}
A large $d$ results in $\boldTheta$ being high-dimensional, leading to $p(\boldx; \boldTheta)$ being over-parameterized. Assuming the true model $q(\boldx) = p(\boldx; \boldTheta^*)$ is sparse with parameter $\boldTheta^*$, the regularized likelihood estimator can effectively estimate the graphical model's parameter. 
\subsection{Score Matching \citep{scorematching1}}
Score matching involves the score functions of the data which are independent of the normalising constant $z(\mb\theta)$.
The classical score matching objective involves taking the expected squared distance between target score $\score{q}$ and model score $\score{\psmall}$, given by
\begin{equation}
    \mathcal{L}_\mathrm{SM}(\mb\theta) = \mathbbE_{q} \brs{ \|\score{q} - \score{\psmall}\|^2 }.
    \label{eq:score_matching_basic1}
\end{equation}
Note that the \cref{eq:score_matching_basic1} is not tractable due to the unknown $\score{q}$. 
\citet{scorematching1} showed that integration by parts enables a tractable version of  \cref{eq:score_matching_basic1} to eliminate any unknown quantities.  

Score matching has been extended for numerous settings, which include \textit{generalised} score matching \citep{scorematching2, yu2019}, applications to high-dimensional graphical models \cite{highdgraphical, pairwisegraphical} and truncated density estimation \citep{song, williams2022}.

\section{Formulation and Motivation}

In this section, we formally introduce the problem of estimating differential parameters in exponential family distributions. 
Let's assume that the true time-varying data generating distribution has a density function $q_t(\mb{x})$.  $q_t(\mb{x})$ is in the exponential family and is parameterized as: 
\begin{equation}
q_t(\mb{x}) = q(\mb{x}; \mb\theta^*(t)) = \frac{\exp \left(\brangle{\mb\theta^*(t), \mb{f}(\mb{x})}\right)
}{z(\mb\theta^*(t))},
\label{eq:q_t}
\end{equation}
where the normalising term $z(\mb\theta^*(t))$ is defined as 
\[
z(\mb\theta^*(t)) := \int \exp \left(\brangle{\mb\theta^*(t), \mb{f}(\mb{x})} \right) \dx, 
\] 
which is generally computationally intractable. 
The density function has a natural parameter $\mb\theta^*(t)$ that changes with $t$. 
We assume that $\partial_t \mb\theta^*(t)$ is a sparse vector, i.e., only a few elements in $\mb\theta^*$ depend on $t$. 
We denote $\mb{f}:\R^d \to \R^k$ as the \textit{feature function}, which does not change over time.

\subsection{Learning Differential Parameters, not Parameters}

In this paper, we want to know how $q_t$ changes over time. 
Specifically, we want to learn the time derivative of the parameter, i.e., $\partial_t \boldtheta^*(t)$.

We could adopt the approach that approximates the density function $q(\mb{x}; \mb\theta^*(t))$ with a parametric model $p(\mb{x}; \hat{\mb\theta}(t))$ 
and subsequently, differentiates the estimated parameter $\hat{\mb\theta}(t)$ with respect to $t$. However, if only a few elements in $\boldtheta^*(t)$ change with $t$, then modelling and learning the full density, including all stationary parameters that do not change with $t$, is inefficient, especially when $\boldtheta^*(t)$ is high dimensional.

\subsection{Modeling Time Score via Differential Parameters}

In the remainder of this paper, we refer to $\partial_t \log q_t(\boldx)$, the time derivative of the log density, as the ``time score function''. 

A key observation that motivated this work is that the time score function can be expressed as a function of the differential parameter $\partial_t \boldtheta(t)$. It follows from~\cref{eq:q_t} that $\partial_t \log q_t(\boldx)$ has a simple closed-form expression given in the following proposition. 

\begin{proposition}\label{thm:score}
The time score function for $q_t$, which is defined in \cref{eq:q_t}, can be written as
\begin{align*}
    \partial_t \log q_t(\boldx)
    = \langle \partial_t\boldtheta^*(t), \boldf(\boldx) \rangle - \E_{\mb y \sim q_t} \left[ \langle \partial_t \boldtheta^*(t), \boldf(\boldy) \rangle \right].
\end{align*}
\end{proposition}
A proof is given in \cref{app:scoret}. \cref{thm:score} shows that $\scoret{q_t}$ does not involve the normalization constant explictly and depends solely on the differential parameter $\partial_t \boldtheta^*(t)$. 
This is an analogue to how the density ratio function only depends on the difference between two parameters $\boldtheta_0 - \boldtheta_1$ as we explain in \cref{sec:back:est.prob.model}. 
Inspired by \cref{thm:score}, we directly model the time score function $\scoret{q_t}$ as 
\begin{equation}
\model \coloneqq \brangle{\partialthetasmall, \mb{f}(\mb{x})} - \E_{\boldy \sim q_t} \brs{\brangle{\partialthetasmall, \mb{f}(\mb{y})}},
\label{eq:model}
\end{equation}
where $\partial_t \boldtheta(t)$ is a parameterization of the differential parameter that we will detail later. 
If we can estimate the model in \cref{eq:model}, we obtain the differential parameter $\partialthetasmall$.

\section{Estimator of Differential Parameters}

\subsection{Time Score Matching}\label{sec:estimating:obj}

Our method seeks to approximate the time score function $\scoret{q_t}$ by $\model$. 
We adopt a learning criterion called time score matching \citep{choi2022}. Specifically, we 
look for $\partial_t \mb\theta(t)$ that minimizes the weighted Fisher-Hyvärinen divergence \citep{Lyu2009}
\begin{equation}
     \mathcal{L}(\partial_t \boldtheta(t)) \coloneqq \timeintegral \Eqtsmall\brs{g(t) \| \scoret{q_t} - \model \|^2} dt.
    \label{eq:deriv:obj1} 
\end{equation} 
\citet{choi2022} uses a neural network to model the time score, which does not allow the direct inference of $\partial_t \boldtheta(t)$. 
However, in our paper, Proposition \ref{thm:score} relates $\partial_t \boldtheta(t)$ to the time score function, enabling the direct estimation of $\partial_t \boldtheta(t)$ using time score matching.

\cref{eq:deriv:obj1} is a truncated score matching problem \citep{song,yu2021}: The variable being differentiated in the log-density function, $t$ is in a bounded domain $[0,1]$. 
The following theorem provides a tractable learning objective without intractable terms such as $\scoret{q_t}$. 
\begin{theorem}
Suppose that $g(t) = 0$ for $t = 0 \text{ or } 1$.
The objective in \eqref{eq:deriv:obj1}
can be written as
\begin{align}
    & \mathcal{L}(\partial_t \boldtheta(t)) = \text{C} + \int_0^1 \Eqtsmall\left[g(t) \model^2 \right]\mathrm{d}t + \notag \\
    &\int_0^1 2 \Eqtsmall\left[  \partial_t g(t) \brangle{\partialthetasmall, \mb{f}(\mb{x})}  + g(t) \brangle {\partial^2_t \boldtheta(t), \mb{f}(\mb{x})}\right]\mathrm{d}t
    \label{eq:deriv:objfinal}
\end{align}
where $C$ is a constant independent of $s(\mb x; t)$.
\label{thm:obj_model}
\end{theorem}
A proof is given in~\cref{app:SMderivation_1}. 
\cref{thm:obj_model} is \emph{not} a starightforward application of Theorem 1 in \citep{choi2022}. The $\E_{\mb y \sim q_t} \left[ \langle \partial_t \boldtheta(t), \boldf(\boldy) \rangle \right]$ term of our score model requires additional steps after the initial integration by parts and the resulting objective is different from Equation (8) in \citep{choi2022}.
Noticing that the weighting technique introduced in \citep{song,yu2021}, we let $g$ be zero at the boundary $t= 0$ and 1 to ensure the boundary condition used in integration by parts is satisfied. 
Moreover, in \citep{choi2022}, authors convert \cref{eq:deriv:obj1} into a different form, which requires multiple samples of $\mb x$ at $t = 0$ and $t = 1$ for Monte Carlo approximation. However, this is not a requirement for our method. The exact form of $g$ can be found in the \cref{sec:weighting.function}. 

Note that this objective function only depends on $\partial_t \boldtheta(t)$, so no nuisance parameters are required when optimizing the above objective. 
To finally transform \cref{eq:deriv:objfinal} into a tractable optimization problem, we need to approximate it using samples and design a parametric model for $\partial_t \boldtheta(t)$. 

\subsection{Sample Approximation of Objective Function}
In this section, we consider sample approximations of \cref{eq:deriv:objfinal} under two different settings. 
In the first scenario, we assume that we have paired samples from the joint distribution of both $\mb x$ and $t$, i.e.  \[\mathcal{D}_1 := \{(t_i, \mb{x}_{i})\}_{i=1}^{n} \sim q(\mb x, t), \]
where $q(\mb x ,t) = q_t(\mb x) \times \mathrm{Uniform}(t; 0, 1)$, where we shortened $q(\mb x | t)$ as $q_t(\mb x)$ and the same below. 
Then the sample objective of \cref{eq:deriv:objfinal} (omitting the constant) can be written as
\begin{align}
    &\frac{1}{{n}} \sum^{n}_{i=1}g(t_{i}) \modelsmallsample{\mb x_i,  t_i}^2 
      + 2 \partial_t g(t_i) \brangle{\partial_t \mb\theta(t_i), \mb{f}(\mb{x}_{i})}  \notag \\
     & +  2 g(t_i) \brangle {\partial_t^2 \mb\theta(t_i), \mb{f}(\mb{x}_{i})}  \label{eq:sampleobj0}.
\end{align}

In the second scenario, we assume that we first draw samples $\{t_j\}_{j = 1}^{m} \sim \mathrm{Uniform}(t; 0, 1)$, then draw samples $\{\mb x_{ij} \}_{i=1}^{n} \sim q_{t_j}(\mb x)$ for each $j \in [m]$. Given the dataset 
\[\mathcal{D}_2 := \{(t_j, \{\mb{x}_{ij}\}_{i = 1}^{n})\}_{j=1}^{m}, \]
the sample objective of \cref{eq:deriv:objfinal} (omitting the constant) is written as
\begin{align}
    &\frac{1}{m}\sum_{j=1}^{m}\frac{1}{{n}} \sum^{n}_{i=1}g(t_{j}) \modelsmallsample{\mb x_{ij},  t_j}^2 
      + 2 \partial_t g(t_j) \brangle{\partial_t \mb\theta(t_j), \mb{f}(\mb{x}_{ij})}  \notag \\
     & +  2 g(t_j) \brangle {\partial_t^2 \mb\theta(t_j), \mb{f}(\mb{x}_{ij})}  \label{eq:sampleobj_tj}.
\end{align}

Both scenarios naturally arise in machine learning tasks. The first scenario resembles a time series setting, where we have a single sample for each time point. In contrast, the second scenario aligns more with a continuous dataset drift setting \citep{quinonero2022dataset}, where we receive a full set of samples for each time point.

\subsection{Sample Approximation of Score Model}

Now we look at the score model used in above objectives 
\[s(\boldx, t) := \brangle{\partialthetasmall, \mb{f}(\mb{x})} - \E_{\boldy \sim q_t} \brs{\brangle{\partialthetasmall, \mb{f}(\mb{y})}}.  \]
From the definition, we can see that $\model$ contains term $\E_{q_{t}} \brs{\brangle{\partialthetasmall, \mb{f}(\mb{x})}}$, which is an \textit{expectation conditioned on} $t$. 
In general, this time-conditional expectation does not have a closed-form expression. Thus, we also have to approximate it using samples. 

In the first scenario, where we only have access to a paired dataset $\mathcal{D}_1$, 
we consider using \textbf{Nadarya-Watson (NW) estimator} \citep{nadaraya1964, watson1964} to estimate the conditional expectation. 
NW estimator is a weighted sample average. In our context, we can write the NW estimator as 
\begin{equation}
\hat{\E}_{q_{t_i}}\brs{\brangle{\partial_t \boldtheta(t_i), \mb{f}(\mb{y})}} := \frac{\sum_{j=1}^m K(t_j, t_i) \brangle{\partial_t \mb \theta(t_j), \mb{f}(\mb{x}_j)}}{\sum_{j=1}^m K(t_j, t_i)},
\label{eq:nw_expectation}
\end{equation}
where $K$ is a Gaussian kernel function. 

In the second scenario, given the dataset $\mathcal{D}_2$, we can simply approximate the expectation as the average of samples obtained at time point $t_j$: \[\hat{\E}_{q_{tj}} \brs{\brangle{\partial_t \boldtheta(t_j), \mb{f}(\mb{y})}} := \frac{1}{n}\sum_{i = 1}^{n} \brangle{\partial_t \mb \theta(t_j), \mb{f}(\mb{x}_{ij})}, \forall j \in [m].\] 

Now we denote the approximated score model as
\[{\hat{s}}(\boldx, t) := \brangle{\partialthetasmall, \mb{f}(\mb{x})} - \hat{\E}_{q_t} \brs{\brangle{\partialthetasmall, \mb{f}(\mb{y})}},  \]
where $\hat{\E}_{q_t} \brs{\brangle{\partialthetasmall, \mb{f}(\mb{y})}}$ is the approximated conditional expectation in the above two settings. 

\subsection{\texorpdfstring{Parameterization of $\partialthetasmall$}{Parameterization of partialtheta}} \label{sec:estimating:partialtheta}

So far, the objective function in \cref{eq:deriv:objfinal} works for any generic parameterisation of $\partialthetasmall$. 
Without the loss of generality, we propose a parametric model 
\[\partial_t \mb\boldtheta(t) := [\partial_t \theta_1(t), \dots, \partial_t \theta_k(t)]^\top, 
\partial_t \theta_j(t) := \langle \mb\alpha_j, \partial_t \mb\phi(t) \rangle,\] 
where $\mb\phi: \mathbb{R} \to \mathbb{R}^b$ is a differentiable basis function and $\mb \alpha_j \in \mathbb{R}^b$ is a parameter vector to be estimated. 

In the simplest case, $\phi(t) := t$, so that $\partial_t \theta_j(t) := \alpha_j$. 
This model implies that $\boldtheta(t)$ changes with $t$ \textit{linearly}. One can consider other basis functions, for example, the Fourier basis, which are often used to model time-dependent functions, 
\begin{align*}
    \mb \phi(t) = [\sin(t), \cos(t), \cdots, \sin(bt/2), \cos(bt/2)]. 
\end{align*}

Suppose we adopt a linear model, i.e., $\phi(t) := t$.  
We can rewrite the time score model as 
\[
\hat{s}_\boldalpha(\boldx, t) := \brangle{\boldalpha, \mb{f}(\mb{x})} - \hat{\E}_{\mb y \sim q_t} \brs{\brangle{\boldalpha, \mb{f}(\mb{y})}}. 
\]
Thus, given the dataset $\mathcal{D}_1$, we have the following tractable  objective: 
\begin{align}
    \hat{\mathcal{L}}(\mb \alpha) := &\frac{1}{{n}} \sum^{n}_{i=1}
    g(t_{i}) \hat{s}_\boldalpha(\mb x_i;  t_i)^2 
      + 2 \partial_t g(t_i) \brangle{\boldalpha, \mb{f}(\mb{x}_{i})} \label{eq:sampleobj}, 
\end{align}

where we have replaced $s$ in Equation \eqref{eq:sampleobj0} with $\hat{s}_\boldalpha$ and $\partialthetasmall$ with the parameter $\boldalpha$. Note that the third term in \eqref{eq:sampleobj0} vanishes since $\partial_t^2 \mb \theta(t) = 0$ given the choice $\phi(t) = t$. 

Define $\mb F \in \mathbb{R}^{n \times k}$ as the feature matrix whose $i$-th row is $\boldf(\boldx_i)$, centered sufficient statistic $\tilde{\mb F} := \boldF - \hat{\E}_{q_t}[\mb F]$, and 
a diagonal matrix $\boldsymbol{G} \in \mathbb{R}^{n\times n}$ whose $i$-th diagonal entry is $g(t_{i})$. We can rewrite \cref{eq:sampleobj} using the following equivalent quadratic form
\begin{align}
    \hat{\mathcal{L}}(\mb \alpha) = \boldalpha^\top \tilde{\mb F}^\top \mb G \tilde{\mb F} \boldalpha /n+ 2 \boldsymbol{1}^\top_n\partial_t \mb G \mb F^\top \boldalpha/n, 
\end{align}
and the minimizer of the above objective has a closed-form expression $\left[ \tilde{\mb F}^\top \mb G \tilde{\mb F} \right]^{-1} \mb F^\top \partial_t \mb G \boldsymbol{1}_n$. However, in the high dimensional setting, where the number of samples $n$ is potentially smaller than the dimension of both $\boldalpha$ and $\boldtheta(t)$, $\tilde{\mb F}^\top \mb G \tilde{\mb F}$ is non-invertible. We introduce a high-dimensional estimator of $\boldalpha$.

\section{High-dimensional Differential Parameter Estimation and  Debiasing}
To simplify our discussion, from now on, we suppose that we have access to dataset $\mathcal{D}_1$ and $\phi(t) = t$. 

In high-dimensional settings, 
we assume only a few elements in $\mb \theta^*(t)$ change over time $t$, making $\partial_t \mb \theta^*(t)$ (and $\mb \alpha$) a sparse vector. Hence, we use a lasso regularizer to identify the non-zeros in $\partial_t \mb \theta(t)$, where we refer to this lasso estimator of $\hat\boldalpha$ as ``SparTSM''. 
We propose minimizing $\hat{\mathcal{L}}(\mb \alpha)$ with a sparsity-inducing $\ell_1$ norm:
\begin{align}
\label{eq:main.obj}
    \hat{\boldalpha} := \argmin_{\mb\alpha} \hat{\mathcal{L}}(\mb \alpha) + \lambda_{\mathrm{lasso}} \|\boldalpha\|_1.
\end{align}
In Section \ref{sec:theory}, we prove the consistency of $\hat\boldalpha$ in a high dimensional setting. %
Using a lasso estimator in \cref{eq:main.obj} can introduce biases, making the asymptotic distribution of $\hat{\boldalpha}$ intractable. This is not ideal if we are interested in parameter inference, such as hypothesis tests and establishing confidence intervals. We apply the debiasing technique \citep{zhang2014confidence, van2014asymptotically} to the lasso estimate of each component, which will allow us to track the asymptotic distribution. Let $\bome^*_j$ denote the $j$-th column of the inverse Hessian $\left[ \nabla^2_{\boldsymbol{\alpha}}\mathcal{L}(\boldsymbol{\alpha}^*) \right]^{-1}$ and $\tilde{\bome}_j$ be a consistent estimator of $\bome^*_j$. We debias the $j$-th element of the lasso estimate using a single-step Newton update:
\begin{equation}\label{eq:debiaslasso}
    \tilde{\ba}_j = \hat{\ba}_j - \tilde{\bome}_j^\top\nabla_{\ba} \hat{\mathcal{L}}(\hat{\boldsymbol{\alpha}}).
\end{equation}
Estimating the inverse Hessian $\left[ \nabla^2_{\boldsymbol{\alpha}}\mathcal{L}(\boldsymbol{\alpha}^*) \right]^{-1}$ in high-dimensional space is challenging, as the empirical Hessian $\nabla^2_{\boldsymbol{\alpha}}\hat{\mathcal{L}}(\hat{\boldsymbol{\alpha}})$ is ill-conditioned and often non-invertible. Fortunately, $\bome^*_j$ satisfies the equality $\left[\nabla^2_{\boldsymbol{\alpha}}\mathcal{L}(\boldsymbol{\alpha}^*)\right]\bome^*_j = \boldsymbol{e}_j$, where $\boldsymbol{e}_j$ is a vector with the $j$-th element equal to one and zeros elsewhere. We estimate $\tilde{\bome}_j$ using the $\ell_1$-norm regularized objective: \begin{equation}\label{loss:invhes}
        \tilde{\bome}_j = \arg\min_{\bome}\frac{1}{2}\bome^\top \nabla^2_{\boldsymbol{\alpha}}\hat{\mathcal{L}}(\hat{\mb \alpha}) \bome - \bome^\top \boldsymbol{e}_j + \lambda_j ||\bome||_1
\end{equation}
The consistency of this estimator has been proved in \cref{app:consiinvhes}. We show in \cref{sec:theory} that the debiased estimator in Equation \eqref{eq:debiaslasso} is asymptotically unbiased and normally distributed under further conditions. We summarize the high-dimensional differential parameter inference pipeline in Algorithm \ref{alg:debiaslasso}. We refer to the estimator of $\tilde{\ba}_j$ in \cref{eq:debiaslasso} as ``SparTSM+''
\begin{algorithm}[t!]
\caption{Inference Pipeline} 
\label{alg:debiaslasso} 
    \begin{algorithmic}[1]
        \REQUIRE Dataset: $\{(t_i, \mb{x}_{i})\}^n_{i=1}$,  \\
        Regularization parameters $\lambda_{\mathrm{lasso}},\lambda_1,...,\lambda_k >0$.\\
        \STATE \text{Find lasso estimator} $\hat{\boldsymbol{\alpha}}$ \text{by solving \eqref{eq:main.obj}}
        \FOR{$j \in [k]$}{
            \STATE \text{Find} $\tilde{\bome}_j$ by solving \eqref{loss:invhes}
            \STATE \text{Obtain debiased lasso by} \eqref{eq:debiaslasso}.
        }
        \ENDFOR\\
        \STATE \textbf{return} asymptotically unbiased estimator $\tilde{\boldsymbol{\alpha}}$
    \end{algorithmic}
\end{algorithm}
\section{Theoretical Analysis}\label{sec:theory}
We show that both SparTSM and SparTSM+ work effectively in a high-dimensional regime when $\partial_t \boldtheta(t)$ is sparse. A full list of notations is provided in \cref{sec:notations}. 
For our theoretical results, we assume the following:
\begin{assumption}
    There exists a unique parameter $\mb\alpha^*$ supported on $S \subseteq \{1, \dots, k\}$ such that $\partial_t \mb\theta^*_t = \mb \alpha^*$, and the population objective $\mathcal{L}(\boldalpha)$ is minimised at $\mb\alpha^*$.
\label{ass:alphastar}
\end{assumption}

\subsection{Finite-sample Estimation Error of SparTSM}

We assume bounded sufficient statistics in the exponential family model \citep{Byol2021, xia2023debiased}.
\begin{assumption}\label{ass:boundedsufficient}
    There exists some constant $ 0< C_{f} < \infty$ such that $\|\boldsymbol{f}(\boldsymbol{x})\|_{\infty}\leq C_f$ almost surely.
\end{assumption}
Defining $\boldsymbol{G}^{1/2} \in \mathbb{R}^{n\times n}$ as a diagonal matrix where $G^{1/2}_{i,i} = \sqrt{g(t_i)}$, we assume the following:
\begin{assumption}
\label{ass:re}
   The matrix $\boldsymbol{G}^{1/2} \tilde{\mb F}$ satisfies the restricted eigenvalue (RE) condition over the support set $S$ with parameters $(\kappa, 3)$, that is, 
    \begin{equation}
        \frac{1}{n} \| \mb G^{1/2} \tilde{\mb F} \Delta\|_2^2  \geq \kappa \|\Delta\|_2^2 \text{ for all } \Delta \in \mathbb{C}_3(S),
    \end{equation}
    where $\mathbb{C}_3(S) = \{\Delta \in \mathbb{R}^k : \|\Delta_{S^c}\|_1 \leq 3\|\Delta_S\|_1\}$.
\end{assumption}
Building on the aforementioned assumption, we can establish a probabilistic upper bound for the $\ell_2$ norm of the error vector $\hat{\boldsymbol{\alpha}} - \boldsymbol{\alpha}^*$. 

Define a random variable 
\begin{align*}
    m_\boldalpha(\mb x, t):= g(t) \hat{s}_\boldalpha(\mb x;  t)^2 + 2 \partial_t g(t) \brangle{\boldalpha, \mb{f}(\mb{x})}  
\end{align*}
where $(\boldx, t) \sim q_t(\mb x)\times \mathrm{Uniform}(t; 0,1)$. 
\begin{theorem}\label{probineqconst}
    Suppose Assumption \ref{ass:alphastar},\ref{ass:boundedsufficient} and \ref{ass:re} hold. Let 
   \[\sigma^2=\max_{1 \leq i \leq k} \boldsymbol{\Sigma}_{ii}(\mb \boldalpha^*), ~~ \boldsymbol{\Sigma}(\boldalpha) := \mathrm{Cov}[\nabla_{\boldalpha} m_\boldalpha(\boldx,t)],\] 
   and we set 
   $
   \lambda_{\mathrm{lasso}}  = 2\sigma\left(\sqrt{\frac{2\log k}{n}}+\delta\right).
   $
   It holds that
    \begin{equation}
        \|\hat{\boldsymbol{\alpha}} - \boldsymbol{\alpha}^*\|_2 \leq \frac{6}{\kappa} \|{\ba}^*\|_0^{1/2} \sigma \left(\sqrt{\frac{2 \log k}{n}} + \delta\right) \text{ for all } \delta > 0.
    \end{equation}
    with probability at least $1-2\exp\{-\frac{n\delta^2}{\sigma}\}$.
\end{theorem}
See \cref{pf:probineqconst} for the proof.
\paragraph{Remarks} Crucially, the error bound only depends on the dimension of the feature function $k$ logarithmically, indicating the estimator indeed scales to high dimensional settings. 
Moreover, the error bound depends on the sparsity of $\boldalpha^*$ (i.e., the sparsity of $\partial_t\boldtheta^*(t)$) and does \emph{not} depend on the sparsity of $\boldtheta (t)$. It means that our method works with a dense parameter vector $\boldtheta (t)$ as long as $\partial_t \boldtheta(t)$ is sparse. This opens the door for applications where time-varying probabilistic models that are complex and cannot be described by sparse models. We showed an example in Section \ref{sec:sim-study}. 

\subsection{Finite-sample Gaussian Approximation Bound of SparTSM+}
In this section, we prove that $\sqrt{n}\left(\tilde{\alpha}_j - \alpha^*_j\right)/\hat{\sigma}_j$ converges to a standard normal random variable.
where $\hat{\sigma}_j$ is defined as
\begin{align}
    \hat{\sigma}^2_j = \tilde{\bome}_j^\top \hat{\mb \Sigma}(\hat{\boldsymbol{\alpha}}) \tilde{\bome}_j, ~~~
    \hat{\mb \Sigma}(\hat{\boldsymbol{\alpha}}) = {\mathrm{Cov}_n}[\nabla_{\boldsymbol{\alpha}} m_{\hat{\boldalpha}} (\boldx, t)], 
\end{align}
and ${\mathrm{Cov}_n}$ is the empirical covariance estimator. 
Now we justify the asymptotic normality using the following Gaussian Approximation Bound (GAB), showing that as the number of samples increases, the distribution of $\tilde{\alpha}_j$ is closing to a Gaussian distribution.
\begin{theorem}[GAB]\label{thm:main}
   Suppose Assumption \ref{ass:alphastar} and \ref{ass:boundedsufficient} hold. 
For $\delta_{\boldsymbol{\alpha}}, \delta_w, \lambda_{\mathrm{lasso}}, \lambda_j, \delta_{\sigma} \in [0,1)$, define the set of events $\boldsymbol{\mathcal{E}}$, 
\begin{align}
    \boldsymbol{\mathcal{E}} = \left\{
    \begin{aligned}
        &\|\nabla_{\boldsymbol{\alpha}} \hat{\mathcal{L}}(\boldsymbol{\alpha}^*)\|_\infty \leq \lambda_{\mathrm{lasso}}/2; \\
        &\|\nabla^2_{\boldsymbol{\alpha}} \hat{\mathcal{L}}(\boldsymbol{\alpha}^*) \bome_j^{*} - \boldsymbol{e}_j\|_\infty \leq \lambda_j/2; \\
        &\|\hat{\boldsymbol{\alpha}} - \boldsymbol{\alpha}^*\|_1 \leq \delta_{\boldsymbol{\alpha}}; \\
        &\|\tilde{\bome}_j - \bome_j^*\|_1 \leq \delta_w; \\
        &||\hat{\mb \Sigma}(\boldsymbol{\alpha}^*) - \boldsymbol{\Sigma}||_\infty \leq \delta_\sigma / 2.
    \end{aligned}
    \right. \notag 
\end{align}
Suppose $\mathbb{P}(\boldsymbol{\mathcal{E}})\geq 1- \epsilon$, denoting $\Phi(\cdot)$ as the cumulative distribution function of the standard normal distribution, we have
\begin{align}
    &\sup_{z \in \mathbb{R}} \left| \mathbb{P}\left\{\sqrt{n} (\tilde{{\alpha}}_j - {\alpha}_j^*)/\hat{\sigma}_j \leq z \right\} - \Phi(z) \right| \nonumber\\
    &\quad \qquad \qquad\qquad\qquad\leq \Delta_1 + \Delta_2 + \frac{\delta_C}{1-\delta_C} + \epsilon
\end{align}
where
\begin{align*}
     \Delta_1 &= 2CM||\bome_j^*||_1/\sqrt{n}\sigma_j, \\
     \Delta_2 &= \sqrt{n}(\lambda_j \delta_{\boldsymbol{\alpha}} + \delta_w\lambda_{\mathrm{lasso}} + K \delta_{{\boldsymbol{\alpha}}}\delta_w)/\bome_j^*\boldsymbol{\Sigma(\boldalpha^*)}\bome_j^*,\\
     \delta_C &= ((2L\delta_\alpha+\delta_\sigma)(||\bome_j^*||^2_1+\delta^2_w) +|||\boldsymbol{\Sigma}|||_\infty\delta_w^2)/\sigma_j^2 
\end{align*}
here, $C, M, K, L$ are fixed constants.
\end{theorem}
See \cref{pf:main} for proof. The proof involves decomposing the standardized debiased estimator into three parts. One part introduces the desired asymptotic normality due to the Berry-Esseen inequality \citep{chen2010normal}, while the other two converge to zero conditioned on $\boldsymbol{\mathcal{E}}$.  
\begin{assumption}\label{REomega}
    The matrix $\boldsymbol{G}^{1/2} \tilde{\mb F}$ satisfies the RE condition over the support set $S_{\bome,j}$ with parameters $(\kappa_j, 3)$ where $S_{\bome,j}$ is the support set of $\bome^*_j$.
\end{assumption}
We can further specify the rate of the approximation under appropriate settings of $\lambda_{\mathrm{lasso}}$ and $\lambda_{j}$. 

\begin{corollary}
    \label{thm:main2}
   Assume \ref{ass:alphastar},\ref{ass:boundedsufficient}, \ref{ass:re} and \ref{REomega} hold. 
    Let $\tilde{\alpha}_j$ be the debiased lasso estimator derived by \cref{eq:debiaslasso} with the regularization parameters set as 
    \begin{equation}
        \lambda_{\mathrm{lasso}} \in \mathcal{O}\left(\sqrt{\frac{\log k}{n}}\right), \text{ }  \lambda_j \in  \mathcal{O}\left(\sqrt{\|\bome^*_j\|_0 \frac{\log k}{n}}\right)
    \end{equation}
    then there exist positive constants $c,c^\prime$ such that
    \begin{align}
         &\sup_{z\in\mathbb{R}} \left|\mathbb{P}\{\sqrt{n}(\tilde{{\alpha}}_j-{\alpha}^*_j)/\hat{\sigma}_j \leq z\}-\Phi(z)\right|\nonumber\\
         &\leq \mathcal{O}\left(\|{{\bome}^*_j}\|_0^{3/2} \|{\ba}^*\|_0{\frac{\log k}{\sqrt{n}}}\right) + c\exp\{-c^\prime \log k\}
    \end{align}
\end{corollary}
See \cref{pf:main2} for the proof. 

\paragraph{Remarks} This result shows that, as sample size increases, the standardized $\alpha_j$ indeed converges to a standard normal variable, allowing us to specify the confidence interval and perform hypothesis tests. Similar to the error bound stated in \cref{probineqconst}, this approximation error bound also only depends on the sparsity of $\boldalpha^*$ (i.e., the sparsity of $\partial_t\boldtheta^*(t)$), rather than the sparsity of $\boldtheta^*(t)$, indicating it is applicable to complex probabilistic models with only sparse changes. We also notice that the result depends on the sparsity of the inverse Hessian column vector, $\bome_j^*$, which is similar to the previous debiased lasso results \citep{Byol2021}.

\section{Related Works}
We now introduce two additional methods for learning the differential parameters in graphical models, which are later compared in the simulation studies. 

\subsection{Estimating Parameter Changes in Probabilistic Models}
\label{sec:para.change.models}
Given two data-generating distributions, $q_0(\boldx) = p(\boldx; \boldtheta_0^*)$ and $q_1(\boldx) = p(\boldx; \boldtheta_1^*)$, we are interested in learning changes in the underlying data-generating distributions, given random samples $\mb x_0 \sim q_0$ and $\mb x_1 \sim q_1$.

One naive way of estimating the parameter change is fitting two probabilistic models $p(\boldx; \boldtheta_0)$ and $p(\boldx; \boldtheta_1)$ from $\mb x_0$ and $\mb x_1$ using lasso estimators, then take the difference of estimated parameters $\hat{\boldtheta}_1 - \hat{\boldtheta}_0$. This approach is sub-optimal, as sparse estimates of $\hat{\boldtheta}_1$ and $\hat{\boldtheta}_0$ do not necessarily lead to sparse estimate of differences. One solution to this problem is to use a ``fused-lasso'' regularizer $\|\boldtheta_0 - \boldtheta_1\|_1$ to encourage the sparsity in changes between two parameters \citep{danaher2014joint}. 

However, applying sparsity inducing norms on individual models assumes that the true probabilistic models $q_0$ and $q_1$ have sparse parameters $\boldtheta_0^*$ and $\boldtheta_1^*$. 
In theoretical analysis, this leads to consistency results depending on the sparsity level of the individual model, the less sparse the individual models are, the worse the convergence rate is (see e.g., Theorem 1 in \citep{yang2012}). 
In reality, $\boldtheta_0^*$ and $\boldtheta_1^*$ may not be sparse, but the difference between them could be sparse. 

To address this issue, previous works propose a density ratio-based approach to directly estimate the difference $\mb \theta^*_0 - \mb \theta^*_1$ \citep{liu2017,Byol2021}. 
The intuition is that the ratio between exponential family models is 
\begin{align*}
    \frac{p(\boldx; \boldtheta_1)}{p(\boldx; \boldtheta_0)} \propto \exp \left( \langle \boldtheta_1 - \boldtheta_0, f(\boldx) \rangle \right), 
\end{align*}
determined entirely by the \emph{differential parameter}. Thus, fitting the density ratio function automatically learns the parameter change. Moreover, since this estimation is completely independent of individual parameters $\boldtheta_0$ and $\boldtheta_1$, we do not need to regularize $\boldtheta_0$ and $\boldtheta_1$, eliminating the sparsity assumptions on $\boldtheta^*_0$ and $\boldtheta^*_1$. 

\subsection{Estimating Time-varying Probabilistic Models}
\label{sec:back:est.prob.model}

While the density ratio approach learns the ``discontinuous change'' from $\boldtheta^*_0$ to $\boldtheta^*_1$, we may be interested in the \emph{continuous process} from $\boldtheta^*_0$ to $\boldtheta^*_1$. 
We are interested in estimating $\boldtheta^*(t)$ given random samples $\mb x_t \sim q_t$. Naturally, one can fit model $p(\mb x; \mb \theta(t))$ to samples $\mb x_t$ at each time point. Assuming the time-varying process is ``sparse'', i.e., only a few parameters change with $t$, 
we can estimate $p(\mb x; \mb \theta(t))$ jointly using a multi-task learning objective with regularization $\| \mb \theta(t) - \mb \theta(t') \|_1$ for a $t'$ that is close to $t$. \citep{Kolar2012,hallac2017network,gibberd2017regularized}.
In this paper, we compare the proposed method with Loggle \citep{yang2018estimating}, which is designed to capture smoothly varying $\mb \Theta(t)$ and is a variant of the above algorithm. 

However, if we are only interested in how $\boldtheta(t)$ changes with time $t$, i.e., 
the time-derivative $\partial_t \boldtheta(t)$, rather than the process $\boldtheta(t)$ itself. 
The aforementioned approach is again sub-optimal, as estimating a high dimensional, over-parameterized $\boldtheta(t)$ requires us to regularize the sparsity of $\boldtheta(t)$ for each $t$, thus again, putting unnecessary sparsity assumptions on $\boldtheta^*(t)$, leading to a consistency results depend on the sparsity level of each $\boldtheta^*(t)$ (see e.g., Theorem 1 in \citep{Kolar2009}.

\section{Simulation Studies}
\label{sec:sim-study}

We evaluate SparTSM and SparTSM+ estimator performance using datasets simulated with Gaussian Graphical Models (GGM). 
We sample $\boldsymbol{x} \sim \mathcal{N}(\boldsymbol{0}, \boldTheta(t)^{-1})$ with $\boldTheta(t) = \boldTheta_0 + \boldTheta'(t)$, where $\boldTheta_0$ is a constant symmetric positive definite dense matrix, and $\boldTheta'(t)$ is a \emph{sparse} symmetric matrix that changes over time. Refer to \cref{app:simulation} for settings of  $\boldTheta'(t)$ in each experiment.

\subsection{Differential Parameters Estimation using SparTSM}
\label{sec:diff.est}

\begin{figure}[t]
    \centering
    \includegraphics[width=0.97\linewidth]{ 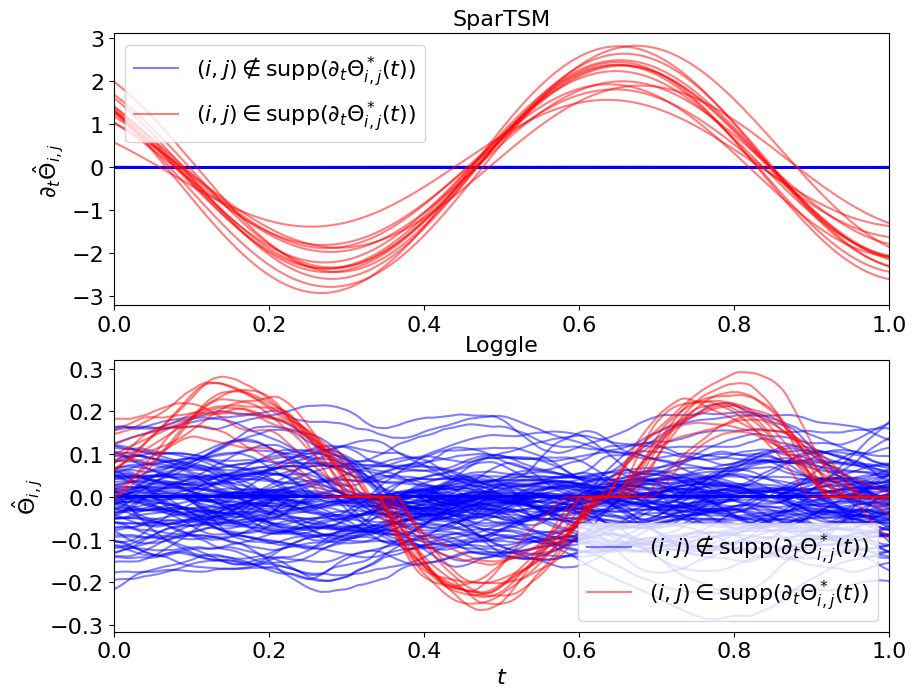}
    \caption{\textbf{SparTSM compared with Loggle.} Estimating the differential parameter vs. Estimating the time-varying parameter. The overall graphical model $\boldTheta$ is non-sparse. }
    \label{fig:tsmvsloggle}
\end{figure}

We generate 5000 synthetic samples from a 20-node Gaussian Graphical Model (GGM) and conduct a performance comparison between Loggle and SparTSM. In this study, the non-zero elements of $\boldTheta'(t)$ are formulated using a sine function (refer to \cref{sec:spartsm}). In \cref{fig:tsmvsloggle}, we illustrate the estimated $\partial_t \Theta_{i,j}(t)$ using SparTSM and the estimated $\Theta_{i,j}(t)$ with Loggle. 
The true time-variant parameters (where $\partial_t \Theta_{i,j}(t) \neq 0$) are marked in red. On one hand, SparTSM, which employs an appropriate basis function (Fourier basis function, see \cref{sec:estimating:partialtheta}), succeeds in accurately estimating differential parameters, thereby clearly identifying the smooth transitions in the time-variant distribution. 
On the other hand, Loggle does identify time-varying parameters (red curves in Figure \ref{fig:tsmvsloggle}), but it  does not immediately discern the changing parameters from the stationary parameters, as its estimates are conflated. 
This aligns with expectations, since our constructed $\mb \Theta(t)$ is consistently \emph{non-sparse}, requiring the estimation of a \emph{dense} parameter vector, which does not adhere to Loggle's sparsity assumptions. In the experiments below, we set $\phi(t) = t$.  
\begin{figure}[t]
    \centering
    \subfigure[\textbf{Time-varying edge detection on a  Gaussian Graphical Model. }The diagonal line has been added for reference.]{
    \includegraphics[width=0.8\linewidth]{ 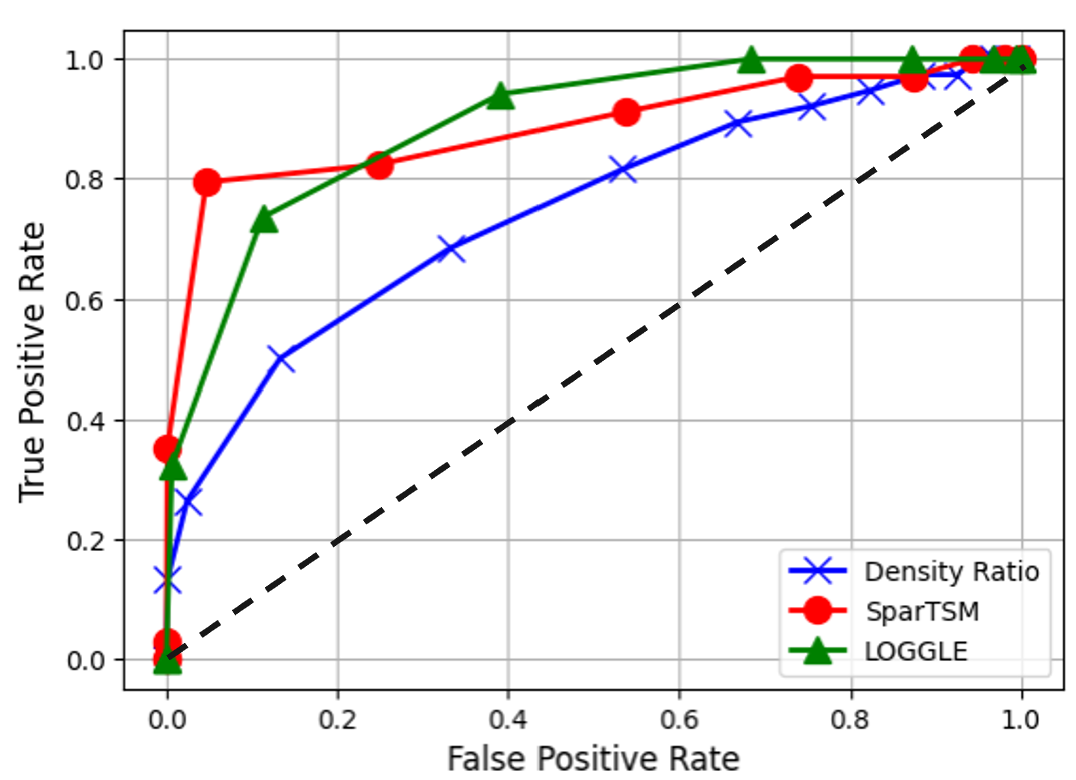}
    \label{fig:roc}
    }   
    \subfigure[\textbf{Time-varying edge detection on a truncated Gaussian Graphical Model. }]{
    \includegraphics[width=0.8\linewidth]{ 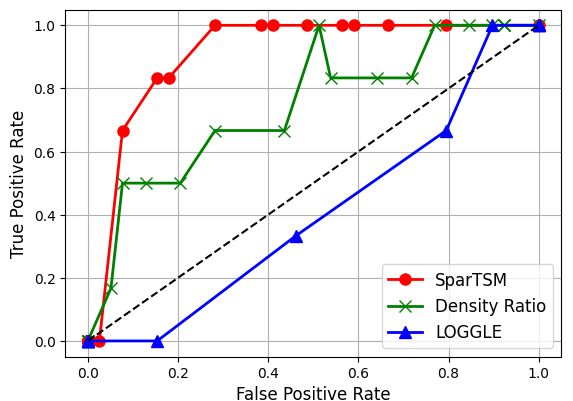}
    \label{fig:roc_tr}
    }
    \caption{ROC curves of SparTSM, Density Ratio, and Loggle.}
\end{figure}
We quantitatively evaluate SparTSM's effectiveness in detecting parameter changes by comparing it to Loggle and the density ratio-based method \citep{liu2017,Byol2021}, using ROC curves for assessment. We construct a GGM that changes linearly over time with 40 nodes and draw 1000 samples (Details can be found in \ref{sec:spartsm}). The task of this experiment is to detect time-varying edges ($\partial_t \Theta_{i,j}(t) \neq 0$) in the GGM. 
For the density ratio approach, $\boldTheta(1) - \boldTheta(0)$ is calculated by estimating the ratio $q_1/q_0$, where we draw 500 samples from $q_0$ and 500 samples from $q_1$. By adjusting the regularization parameter in both SparTSM and the density ratio approach, we can modify the number of changes detected, resulting in an ROC curve. Note that Loggle does not directly provide $\partial_t \Theta_{i,j}(t)$, but rather offers a timeline of $\Theta_{i,j}(t)$ values. Since the true $\Theta_{i,j}(t)$ follows a linear trend, we estimate the time derivative using the least squares regression coefficients (see \cref{loggle_detail}). Different detection thresholds then yield a series of sensitivity levels. Note that this trick can only be applied when we know that the underlying change is linear.
The ROC curves are displayed in \cref{fig:roc}, demonstrating that SparTSM's ROC curve achieves a performance comparable to Loggle and both SparTSM and Loggle significantly outperforms the density ratio method. We compute the average AUC value over 10 random trials in Tabel \ref{tab:AUC}, and it can be seen that Loggle marginally outperforms SparTSM. 

We restrict the domain of the Gaussian Graphical Model (GGM) to $\mathbb{R}_+^{10}$ and generate 2,000 samples from this truncated GGM \citep{lin2016estimation}. The ROC curves of the three methods are shown in \cref{fig:roc_tr}, where the results demonstrate that both SparTSM and density ratio exhibit superior performance, while Loggle fails completely (AUC $\approx 0.5$). This failure arises because the truncated GGM density contains an intractable normalizing constant, rendering the likelihood function in Loggle invalid. The density ratio and SparTSM method do not invovle normalizing constant calculation, and is not affected by the intractable normalising constant. 
The average AUC values in \cref{tab:AUC} further confirm this trend, with the SparTSM method significantly outperforming both density ratio and Loggle.

\begin{table}
\centering
\resizebox{1.0\linewidth}{!}{
\begin{tabular}{lccc}
\toprule
 & SparTSM & Density Ratio & Loggle \\
\midrule
GGM & 0.875 (0.025) & 0.738 (0.150) & \textbf{0.893 (0.026)} \\ 
Trunc. GGM & \textbf{0.733 (0.122)} & 0.624 (0.153) & 0.486 (0.120) \\ 
\bottomrule
\end{tabular}
}
\caption{Comparison of average AUCs over 10 trials. }
\label{tab:AUC}
\end{table}

\subsection{Differential Parameter Inference using SparTSM+}
\begin{table}[t]
\centering
\resizebox{0.7\linewidth}{!}{
\begin{tabular}{lcc}
\toprule
Method & Deterministic & Random \\
\midrule  
Loggle & 4.0\% & 6.0\% \\ 
\midrule              
Oracle & 3.4\% & 2.2\% \\ 
\bottomrule
\rowcolor[gray]{0.9}
\textbf{SparTSM+} & \textbf{5.6\%} & \textbf{5.3\%} \\ 
\bottomrule
\end{tabular}
}
\caption{The proportions of unsuccessful coverage at nominal confidence level of $95\%$. $\mathcal{H}_0:\partial_t\Theta_{1,2}(t) = 0$.}
\label{tab:dlasso}
\end{table}
 
\label{sec:spartsmplus}
\begin{figure}[t]
\centering
\includegraphics[width=0.9\linewidth]{ 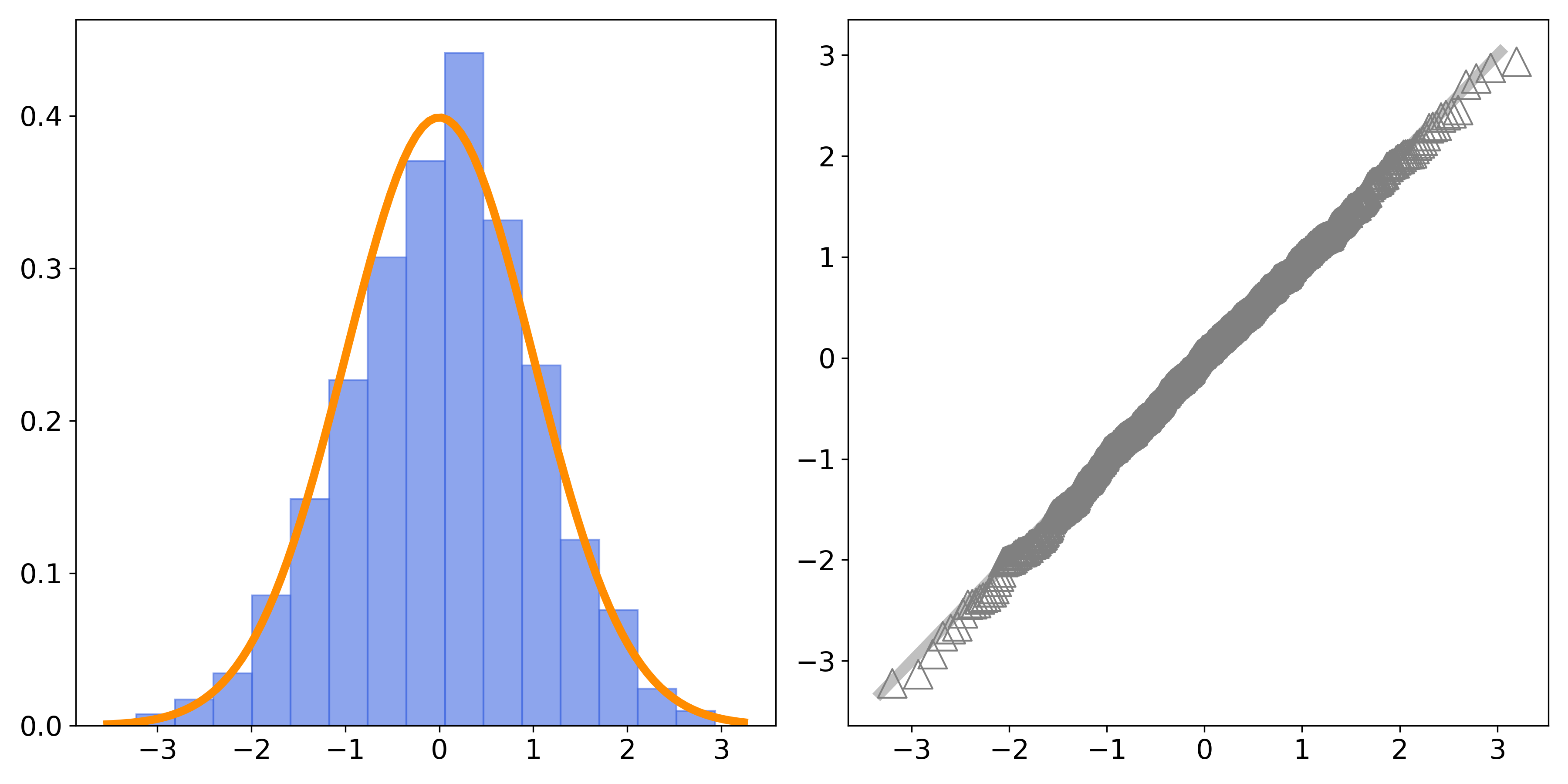}
\caption{\textbf{Gaussian approximation }of SparTSM+. The left shows a bar plot and a QQ-plot is on the right.}
\label{fig:qqplot}
\end{figure}

\begin{figure}[t]
\centering
\includegraphics[width=0.6\linewidth]{ 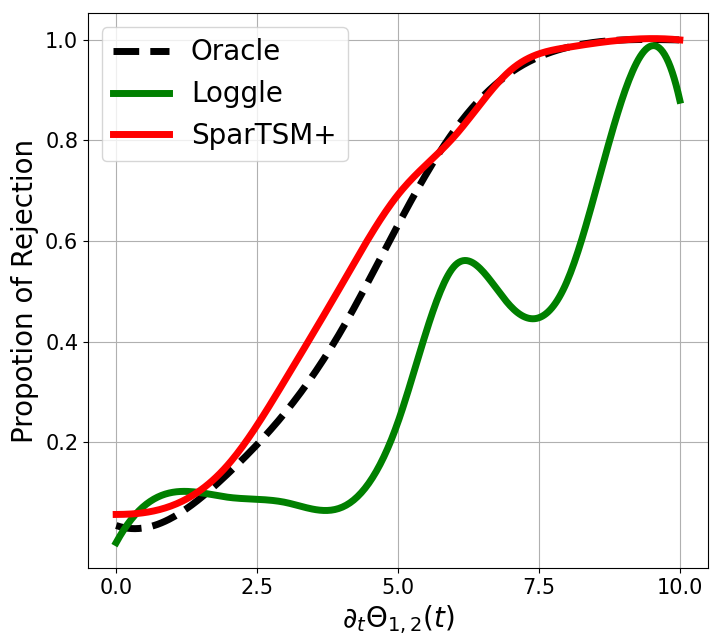}
\caption{\textbf{Testing Power plot}. $\mathcal{H}_0 : \partial_t\Theta_{1,2}(t) = 0$. }
\label{fig:power}
\end{figure}

In exploring SparTSM+, we assess it using two distinct linear GGM datasets, with details provided in \cref{app:expset}. We create 400 samples and 20 nodes from both fixed and random precision matrices, run SparTSM+ 1000 times, and plot the distribution of $(\tilde{\alpha}_{1,2} - \alpha^*_{1,2}) / {\hat{\sigma}_{1,2}}$ in \cref{fig:qqplot}. This showcases the effectiveness of the Gaussian approximation for the standardized SparTSM+. The histogram closely aligns with the standard normal density function. In the Q-Q plot, data points align with or are near the reference line, underscoring the precision of the Gaussian approximation and validating \cref{thm:main,thm:main2} through our experimental results. 

We further compare SparTSM+ against the Oracle method and Loggle regarding confidence interval coverage across 1000 iterations. The Oracle approach presupposes known sparsity of elements and constructs confidence intervals using the asymptotic variance of an M-estimator \citep{Vaart_1998}. Loggle's confidence intervals are derived from the $2.5\%$ and $97.5\%$ quantiles of the test statistic (estimated slope) from 100 permutation tests.

\cref{tab:dlasso} shows the proportions of failure in achieving the nominal confidence level of $95\%$. Even with limited sample sizes, SparTSM+ maintains coverage close to the intended level, whereas the oracle is more conservative. Notably, our heuristic method for determining the confidence interval for Loggle also achieves fairly reliable coverage. Nonetheless, the permutation test is computationally intensive in practice.

As illustrated in \cref{fig:power}, the power of tests is demonstrated for $\partial_t\Theta_{1,2}(t)$ values between 0 and 10. The power is calculated as the ratio of rejections across 1000 independent trials at a significance level of $0.05$ under the null hypothesis $\mathcal{H}_0:\partial_t\Theta_{1,2}(t) = 0$. It is evident that as $\partial_t\Theta_{1,2}(t)$ varies from $\mathcal{H}_0$, the rejection rate of our proposed method escalates, showcasing the test's efficacy. Notably, the proposed method usually provides greater power than the oracle, emphasizing its effectiveness in high-dimensional contexts.

\section{Application: 109th US Senate}
\begin{figure}[t]
    \centering
    \includegraphics[width=0.99\linewidth]{ 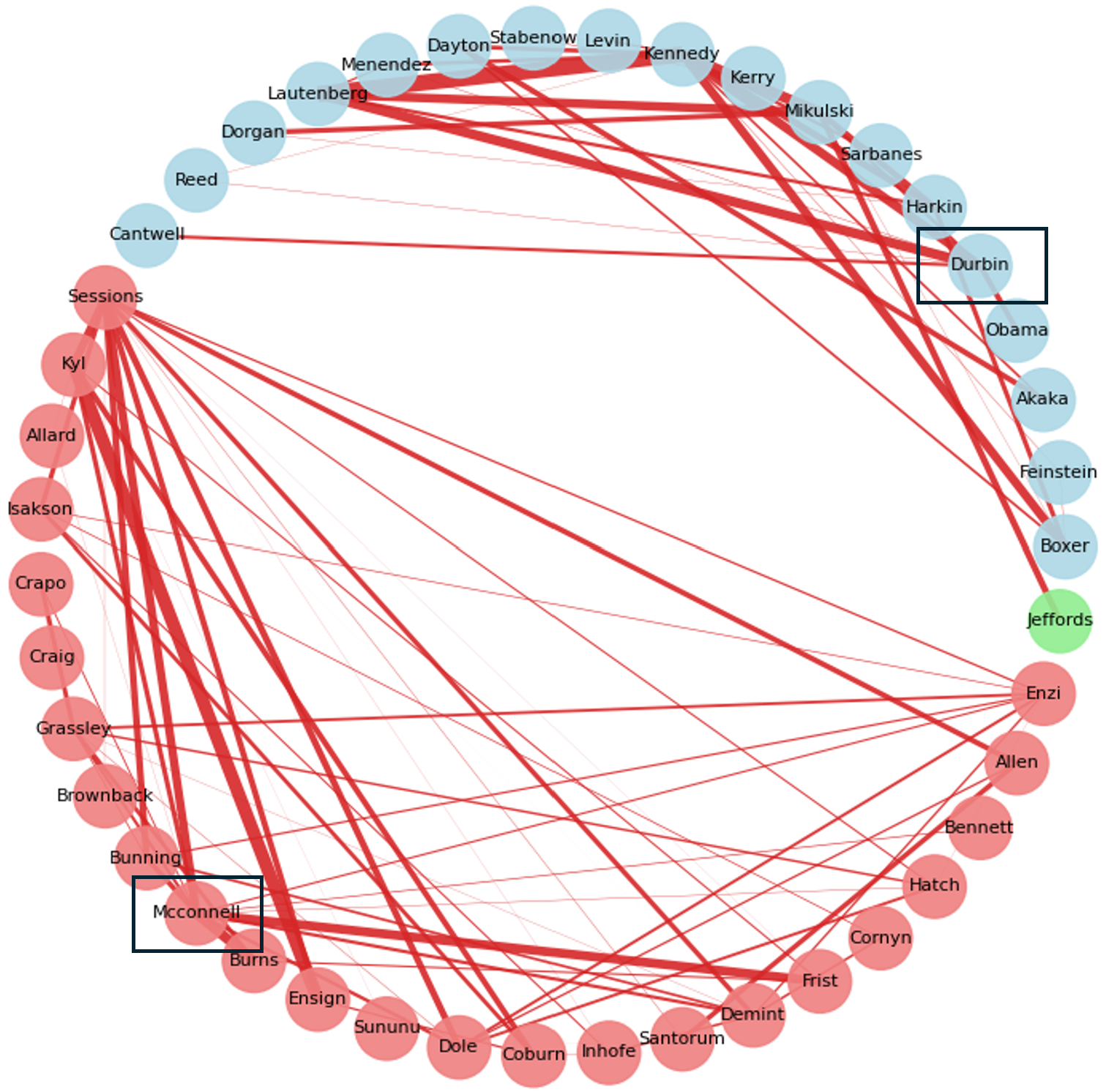}
    \caption{\textbf{The differential graph estimated from 109th senate voting dataset. }Red edge indicate a positive $\alpha_{i,j} \approx \partial_t \Theta^*_{i,j}(t)$. The widths of edges are proportional to $|\alpha_{i,j}|$.
    Democrats, republicans, independents are colored in blue, red and green respectively. Party whips are marked by a black rectangle.}
    \label{fig:senate}
\end{figure}
We employ SparTSM on the voting data from the 109th US Senate \citep{Roy2016Change}, which encapsulates the choices made by 100 senators over around two years. The data is organized as $\mathcal{D} = \{(t_i, \mb x_i)\}_{i=1}^{n=427}, \mb x_i\in \{0,1\}^{100},$ where `1' signifies a yea vote and `0' represents a nay vote. We consider these votes to adhere to a pairwise, time-dependent Ising model, described as $q_t \propto \exp\left( \sum_{i,j} \Theta_{i,j}(t) x_i x_j\right)$, and utilize SparTSM to estimate the differential parameter assuming a linear model $\Theta_{i,j}(t) = \alpha_{i,j}t$. The parameter $\lambda_{\mathrm{lasso}}$ is adjusted to ensure fewer than 100 non-zero elements remain in $\boldalpha$. In \cref{fig:senate}, we illustrate the differential graph $G = (V, E)$, where $E := \{(V_i, V_j) | \hat{\alpha}_{i,j} \neq 0 \}$, meaning the edges denote variations in pairwise interactions within the Ising model. Nodes without any connection are excluded.

Notably, all calculated $\hat{\alpha}_{i,j}$ values are positive and occur within the same party. This suggests that as the congressional term advances, senators increasingly align their votes with key party figures (like whips), creating ``voting blocks'' within the party. Furthermore, there is no apparent bipartisanship emerging between parties. In summary, these findings demonstrate a rise in partisanship throughout the congressional term.

\section{Limitations and Future Works}
Although the proposed method has strong theoretical properties and performs well in simulations, it has several limitations.
First, it can only directly estimate the time differential parameter \emph{within an exponential family}, as general probabilistic models cannot be parameterized solely using differential parameters.
Second, the sufficient statistics $\boldf$ must be specified in advance, and selecting them for real-world datasets remains an open challenge.
Third, the theoretical guarantees for the debiased estimator rely on the sparsity of the inverse Hessian of the objective function—an assumption also made in \citet{Byol2021} but not necessarily valid for all probabilistic models.

A promising future direction is to extend this work to settings where the sufficient statistics $\boldf$ can evolve over time, allowing them to be learned in a way that better captures the complexity of real-world data.

\section*{Acknowledgments}
Daniel J. Williams was supported by the EPSRC Centre for Doctoral Training in Computational Statistics and Data Science, grant number EP/S023569/1. This work was carried out while Leyang Wang was on a University of Bristol School of Mathematics undergraduate summer
bursary placement funded by the Heilbronn Institute for Mathematical Research. The research of Mladen Kolar is supported in part by NSF ECCS-2216912.

\bibliography{aistats_paper}
\newpage
\appendix
\onecolumn

\section{Notations} 
\label{sec:notations}
We denote $\|\cdot\|$ as a norm on $\mathbb{R}^k$ and use $\|\boldsymbol{v}\|_l = \left({\sum_{i=1}^k |v_i|^l}\right)^{1/l}$ as an usual $\ell_l$ norm for $l \in [1,\infty]$ and 
$\|\boldsymbol{v}\|_0 = |\text{supp}(\boldsymbol{v})| = |\{i: v_i \neq 0\}|$; for a matrix $\boldsymbol{M} \in \mathbb{R}^{k\times k}$, $\|\boldsymbol{M}\| = \|\text{Vec}(\boldsymbol{M})\|$; $\text{For } l>0$, $|||\boldsymbol{M}|||_l = \sup_{\|v\|_0 \leq l, \|v\| = 1} |v^\top \boldsymbol{M}v|$ is the maximum $ l$-sparse eigenvalue of $\boldsymbol{M}$; consequently, $|||\boldsymbol{M}|||_\infty = \sup_{\|\boldsymbol{v}\|\leq 1} \|\boldsymbol{Mv}\|_\infty$.

\section{Proof of Theorem \ref{thm:score}} \label{app:scoret}

Recall that 
\[
q(\mb{x}; \mb\theta^*(t)) := \frac{\exp \left(\brangle{\mb\theta^*(t), \mb{f}(\mb{x})}\right)
}{z(\mb\theta^*(t))}.
\]
Thus,
\begin{align*}
     \partial_t \log q(\mb{x}; \mb\theta^*(t))  & = \partial_t \brangle{\mb\theta^*(t), \mb{f}(\mb{x})}  - \partial_t \log z(\mb\theta^*(t)) \\
     &= \brangle{\partial_t\mb\theta^*(t), \mb{f}(\mb{x})} - \frac{ \partial_t  z(\mb\theta^*(t))}{z(\mb\theta^*(t))} \\
     &= \brangle{\partial_t\mb\theta^*(t), \mb{f}(\mb{x})}  - \frac{ \partial_t  \xintegral \exp \left(\brangle{\mb\theta^*(t), \mb{f}(\mb{y})} \right)d\mb{y} }{z(\mb\theta^*(t))} \\
     &= \brangle{\partial_t\mb\theta^*(t), \mb{f}(\mb{x})} - \frac{   \xintegral \partial_t \brangle{\mb\theta^*(t), \mb{f}(\mb{y})} \exp \left(\brangle{\mb\theta^*(t), \mb{f}(\mb{y})} \right)d\mb{y}}{z(\mb\theta^*(t))} \\
     &= \brangle{\partial_t\mb\theta^*(t), \mb{f}(\mb{x})} - \frac{   \xintegral \brangle{\partial_t\mb\theta^*(t), \mb{f}(\mb{y})}\exp \left(\brangle{\mb\theta^*(t), \mb{f}(\mb{y})} \right) d\mb{y} }{z(\mb\theta^*(t))} \\
     &= \brangle{\partial_t\mb\theta^*(t), \mb{f}(\mb{x})} -    \xintegral \brangle{\partial_t\mb\theta^*(t), \mb{f}(\mb{y})}\frac{\exp \left(\brangle{\mb\theta^*(t), \mb{f}(\mb{y})} \right)}{z(\mb\theta^*(t))}d\mb{y}  \\
     &= \brangle{\partial_t\mb\theta^*(t), \mb{f}(\mb{x})} -    \xintegral \brangle{\partial_t\mb\theta^*(t), \mb{f}(\mb{y})}q(\mb{y}; \mb\theta^*(t))d\mb{y} \\
     &= \brangle{\partial_t\mb\theta^*(t), \mb{f}(\mb{x})} -    \Eqtsmall\brs{ \brangle{\partial_t\mb\theta^*(t), \mb{f}(\mb{y})}},
\end{align*}
as desired. 

\section{Proof of \cref{thm:obj_model}} \label{app:SMderivation_1}

Let us begin from the initial formulation of our time based score matching objective with weight function $g(t) = g(t)$ for which $g(0) = g(1) = 0$, i.e. $g(t) = 0$ at the edges of our time domain $t \in [0, 1]$. The initial objective is given by
\begin{align*}
\mathcal{L}(\partial_t \boldtheta(t)) &= \timeintegral \Eqtsmall\brs{g(t) \| \scoret{q_t} - \modelsmall \|^2} dt \\
     &= \timeintegral \Eqtsmall\brs{g(t) \modelsmall^2} dt - 2 \timeintegral\Eqtsmall\brs{g(t) \modelsmall \scoret{q_t}}dt + \timeintegral\Eqtsmall\brs{g(t)\scoret{q_t}^2} dt \\
     &= \timeintegral \Eqtsmall\brs{g(t) \modelsmall^2} dt - 2 \timeintegral\xintegral \brs{q_t \scoret{q_t}} g(t) \modelsmall  d\mb{x}dt + \timeintegral\Eqtsmall\brs{g(t)\scoret{q_t}^2} dt \\
     &= \timeintegral \Eqtsmall\brs{g(t) \modelsmall^2} dt - 2 
     \timeintegral\xintegral \partial_t q_t g(t) \modelsmall  d\mb{x}dt + \timeintegral\Eqtsmall\brs{g(t)\scoret{q_t}^2} dt, 
\end{align*}
where in the final line we have used $q_t \scoret{q_t} = q_t \partial_t \log q_t = \partial_t q_t$ to simplify. First note that the final term is a constant with respect to the model $\modelsmall$, and so we can write $C = \timeintegral\Eqtsmall\brs{g(t)\scoret{q_t}^2} dt$. We continue by expanding the middle term via integration by parts
\begin{align}
\mathcal{L}(\partial_t \boldtheta(t)) &= \timeintegral \Eqtsmall\brs{g(t) \modelsmall^2} dt - 2 \brs{ \xintegral q_t g(t) \modelsmall d\mb{x}}^{t=1}_{t=0} + 2 \timeintegral \xintegral q_t \partial_t (g(t) \modelsmall) d\mb{x} dt + C \nonumber \\
      &= \timeintegral \Eqtsmall\brs{g(t) \modelsmall^2} dt + 2 \timeintegral \xintegral q_t \partial_t (g(t) \modelsmall) d\mb{x} dt + C \nonumber \\
      &= \timeintegral g(t) \Eqtsmall\brs{\modelsmall^2} dt + 2 \timeintegral \partial_t g(t) \Eqtsmall \brs{\modelsmall} dt + 2\timeintegral g(t) \Eqtsmall\brs{\partial_t \modelsmall} dt + C \label{eq:app:deriv1},
\end{align}
where the second equality is due to $\brs{ \xintegral q_t g(t) \modelsmall d\mb{x}}^{t=1}_{t=0} = 0$ as $g(0) = g(1) = 0$. 

Let us now substitute the model given in \cref{eq:model}, stated again here for completeness, given by
\[
\modelsmall = \model = \brangle{\partialthetasmall, \mb{f}(\mb{x})} - \Eqtsmall \brs{\brangle{\partialthetasmall, \mb{f}(\mb{y})}}.
\]
We calculate $\Eqtsmall\brs{\modelsmall}$ and $\Eqtsmall\brs{\partial_t \modelsmall}$ to substitute into the equation above. These are given by
\begin{align*}
    \Eqtsmall\brs{\modelsmall} &= \Eqtsmall\brs{\brangle{\partialthetasmall, \mb{f}(\mb{x})} - \Eqtsmall \brs{\brangle{\partialthetasmall, \mb{f}(\mb{y})}}} \\
    &= \Eqtsmall\brs{\brangle{\partialthetasmall, \mb{f}(\mb{x})}} - \Eqtsmall \brs{\brangle{\partialthetasmall, \mb{f}(\mb{y})}} \\
    &= 0 \\
    \Eqtsmall\brs{\partial_t \modelsmall} &= \Eqtsmall\brs{\partial_t \br{\brangle{\partialthetasmall, \mb{f}(\mb{x})} - \Eqtsmall \brs{\brangle{\partialthetasmall, \mb{f}(\mb{y})}}}} \\
    &= \Eqtsmall\brs{\brangle{\partialdthetasmall, \mb{f}(\mb{x})} - \partial_t \br{\Eqtsmall \brs{\brangle{\partialthetasmall, \mb{f}(\mb{y})}}}} \\
    &= \Eqtsmall\brs{\brangle{\partialdthetasmall, \mb{f}(\mb{x})} - \partial_t \br{\xintegral q_t \brangle{\partialthetasmall, \mb{f}(\mb{y})} d\mb{y}}} \\
    &= \Eqtsmall\brs{\brangle{\partialdthetasmall, \mb{f}(\mb{x})} - \xintegral\partial_t q_t \brangle{\partialthetasmall, \mb{f}(\mb{y})} d\mb{y} - \xintegral q_t \brangle{\partialdthetasmall, \mb{f}(\mb{y})} d\mb{y} } \\
    &= \Eqtsmall\brs{\brangle{\partialdthetasmall, \mb{f}(\mb{x})} - \xintegral\partial_t q_t \brangle{\partialthetasmall, \mb{f}(\mb{y})} d\mb{y} - \Eqtsmall{ \brangle{\partialdthetasmall, \mb{f}(\mb{y})} } } \\ 
    &= \Eqtsmall\brs{\brangle{\partialdthetasmall, \mb{f}(\mb{x})} - \xintegral\partial_t q_t \brangle{\partialthetasmall, \mb{f}(\mb{y})} d\mb{y}} - \Eqtsmall \brs{ \brangle{\partialdthetasmall, \mb{f}(\mb{y})}  } \\
    &= \Eqtsmall\brs{ - \xintegral\partial_t q_t \brangle{\partialthetasmall, \mb{f}(\mb{y})} d\mb{y} },
\end{align*}
where the integral $\xintegral(\dots)d\mb{y}$ is the same integral as $\xintegral(\dots)d\mb{x}$ but written as such to make them distinct from one another. Substituting these into \cref{eq:app:deriv1} gives
\begin{align*}
\mathcal{L}(\partial_t \boldtheta(t)) &= \timeintegral g(t) \Eqtsmall\brs{\modelsmall^2} dt - 2\timeintegral g(t) \Eqtsmall\brs{ \xintegral \partial_t q_t \brangle{\partialthetasmall, \mb{f}(\mb{y})} d\mb{y} } dt + C \\
      &= \timeintegral g(t) \Eqtsmall\brs{\modelsmall^2} dt - 2\timeintegral g(t) \xintegral q_t \br{  \xintegral \partial_t q_t \brangle{\partialthetasmall, \mb{f}(\mb{y})} d\mb{y}}d\mb{x}  dt + C \\
      &\myeq{(a)} \timeintegral g(t) \Eqtsmall\brs{\modelsmall^2} dt - 2\timeintegral g(t) \br{\xintegral q_t d\mb{x}}  \xintegral \partial_t q_t \brangle{\partialthetasmall, \mb{f}(\mb{y})} d\mb{y} dt + C \\
      &\myeq{(b)} \timeintegral g(t) \Eqtsmall\brs{\modelsmall^2} dt - 2\timeintegral g(t) \xintegral \partial_t q_t \brangle{\partialthetasmall, \mb{f}(\mb{x})} d\mb{x} dt + C. 
\end{align*}
In the equality denoted by (a), we have used the fact that inside the integral $\xintegral q_t (\dots) d\mb{x}$ the only variable dependent on $\mb{x}$ was $q_t$, as $\mb{x}$ and $\mb{y}$ are independent. We also have that $\xintegral q_t d\mb{x} = 1$ by $q_t$ being a probability density function. The equality denoted by (b) contains a re-labelling of $\mb{x} = \mb{y}$, as they are the same variable, only labelled differently originally to make them distinct. We use integration by parts one final time to obtain
\begin{align*}
\mathcal{L}(\partial_t \boldtheta(t)) &= \timeintegral g(t) \Eqtsmall\brs{\modelsmall^2} dt - 2 \brs{\xintegral q_t g(t) \brangle{\partialthetasmall, \mb{f}(\mb{x})} d\mb{x}}^{t=1}_{t=0} + 2 \timeintegral \xintegral q_t \partial_t \br{g(t) \brangle{\partialthetasmall, \mb{f}(\mb{x})}} d\mb{x} dt + C \\
      &= \timeintegral g(t) \Eqtsmall\brs{\modelsmall^2} dt + 2 \timeintegral \xintegral q_t \partial_t \br{g(t) \brangle{\partialthetasmall, \mb{f}(\mb{x})}} d\mb{x} dt + C, 
\end{align*}
using again that $g(0) = g(1) = 0$, and finally
\begin{align*}
\mathcal{L}(\partial_t \boldtheta(t)) &= \timeintegral g(t) \Eqtsmall\brs{\modelsmall^2} dt + 2 \timeintegral \xintegral q_t \partial_t g(t) \brangle{\partialthetasmall, \mb{f}(\mb{x})} d\mb{x} dt + 2\timeintegral\xintegral q_t g(t) \brangle{\partialdthetasmall, \mb{f}(\mb{x})} d\mb{x} dt + C \\
      &= \timeintegral \Eqtsmall \brs{g(t) \modelsmall^2 + 2 \partial_t g(t) \brangle{\partialthetasmall, \mb{f}(\mb{x})} + 2 g(t) \brangle{\partialdthetasmall, \mb{f}(\mb{x})}} dt
\end{align*}
which is the same as \cref{eq:deriv:objfinal} in the main text.
\section{Finite-sample Estimation Error of Lasso Estimator}

An equivalent sample objective function is as the following:
\begin{equation}
     \hat{\mathcal{L}}(\boldsymbol{\alpha}) =  \frac{1}{2n}\left[\boldsymbol{\alpha} \tilde{\boldsymbol{F}}^\top \boldsymbol{G}\tilde{\boldsymbol{F}} \boldsymbol{\alpha}^\top+
    2\boldsymbol{1}_n^\top \partial_t \boldsymbol{G}{\boldsymbol{F}}\boldsymbol{\alpha}^\top \right]\label{lafeq} 
\end{equation}
where diagonal matrix $\boldsymbol{G},\partial_t \boldsymbol{G}\in \mathbb{R}^{n\times n}$ with $i$-th diagonal entry to be  $g(t_{i})$ and $\partial_t g(t_i)$ respectively.

\begin{theorem} \label{consi}
    Suppose Assumption \ref{ass:alphastar} and \ref{ass:re} hold. Any minimizer of the objective function \cref{eq:main.obj} with regularization parameter lower
    bounded as $\lambda_{\mathrm{lasso}} \geq  2\left\|\nabla_{\boldsymbol{\alpha}} \hat{\mathcal{L}}(\boldsymbol{\alpha}^*)\right\|_{\infty}$ satisfies
    \begin{equation}
        ||\hat{\boldsymbol{\alpha}} - \boldsymbol{\alpha}^*||_2 \leq \frac{3}{\kappa}{\|{\ba}^*\|_0^{1/2}}\lambda_{\mathrm{lasso}}, ~~~
        ||\hat{\boldsymbol{\alpha}} - \boldsymbol{\alpha}^*||_1 \leq \frac{6}{\kappa}\|{\ba}^*\|_0\lambda_{\mathrm{lasso}}.
        \label{firstineqconst}
    \end{equation}
\end{theorem}
\subsection{Proof of Theorem \ref{consi}}\label{consistency}
Following equation \ref{lafeq} and Assumption \ref{ass:alphastar}, we can express the second order Taylor polynomial around $\boldsymbol{\alpha}^*$ as follows:
\begin{equation}
    \hat{\mathcal{L}}(\hat{\boldsymbol{\alpha}}) = \hat{\mathcal{L}}(\boldsymbol{\alpha}^*) + \frac{1}{n} \left[ 
\boldsymbol{\alpha}^*{\tilde{\boldsymbol{F}}}^\top \boldsymbol{G}{\tilde{\boldsymbol{F}}}\boldsymbol  +  \boldsymbol{1}_n^\top \partial_t \boldsymbol{G} {\boldsymbol{F}}\right](\boldsymbol{\hat{\alpha}} - \boldsymbol{\alpha}^*)^\top+ \frac{1}{2n}(\boldsymbol{\hat{\alpha}} - \boldsymbol{\alpha}^*) \left[ 
{\tilde{\boldsymbol{F}}}^\top \boldsymbol{G}{\tilde{\boldsymbol{F}}}\right](\hat{\boldsymbol{\alpha}} - \boldsymbol{\alpha}^*)^\top
\end{equation}
where the residual $R_2(\hat{\boldsymbol{\alpha}}, \boldsymbol{\alpha}^*) = 0$ since $\hat{\mathcal{L}}$ quadratic.

\begin{lemma}
    Under condition $\lambda_{\mathrm{lasso}} \geq 2 \left\|\nabla_{\boldsymbol{\alpha}} \hat{\mathcal{L}}(\boldsymbol{\alpha}^*)\right\|_\infty$, the error vector $\widehat{\Delta} = \hat{\boldsymbol{\alpha}}- \boldsymbol{\alpha}^* \in \mathbb{C}_{3}(S)$
\end{lemma}
\begin{proof}
Since $\hat{\boldsymbol{\alpha}}$ is optimal, we have 
\begin{equation}
        \hat{\mathcal{L}}(\hat{\boldsymbol{\alpha}}) + \lambda_{\mathrm{lasso}} ||\hat{\boldsymbol{\alpha}}||_1 \leq \hat{\mathcal{L}}(\boldsymbol{\alpha}^*) + \lambda_{\mathrm{lasso}} ||{\boldsymbol{\alpha}}^*||_1
\end{equation}
Rearranging we have from second order Taylor approximation:
\begin{equation}
    0 \leq \frac{1}{2n}||\boldsymbol{G}^{\frac{1}{2}}{\tilde{\boldsymbol{F}}}\hat{\Delta}^\top||_2^2 \leq \frac{1}{n} |\nabla_{\boldsymbol{\alpha}}\hat{\mathcal{L}}(\boldsymbol{\alpha}^*)\hat{\Delta}^\top| + \lambda_{\mathrm{lasso}}\{||\boldsymbol{\alpha}^*||_1 - ||\hat{\boldsymbol{\alpha}}||_1\}
\end{equation}
where $|\cdot|$ denote the absolute value. Now since $\boldsymbol{\alpha}^*$ is $S$-sparse, we can write 
\begin{equation}
\|\boldsymbol{\alpha}^*\|_1-\|\hat{\boldsymbol{\alpha}}\|_1=\|\boldsymbol{\alpha}_S^*\|_1-\|\boldsymbol{\alpha}_S^*+\widehat{\Delta}_S\|_1-\|\widehat{\Delta}_{S^c}\|_1.
\end{equation}
Using Holder’s inequality and the triangle inequality, we have
\begin{align}
    0 &\leq \frac{1}{n}||G^{\frac{1}{2}}{\tilde{\boldsymbol{F}}}\hat{\Delta}^\top||_2^2 \leq 2|\nabla_{\boldsymbol{\alpha}}\hat{\mathcal{L}}(\boldsymbol{\alpha}^*)\hat{\Delta}^\top| + 2\lambda_{\mathrm{lasso}} \{ \|\boldsymbol{\alpha}_S^*\|_1 - \|\boldsymbol{\alpha}^*_S+\widehat{\Delta}_S\|_1 - \|\widehat{\Delta}_{S^c}\|_1 \} \\
    &\leq  2\|\nabla_{\boldsymbol{\alpha}}\hat{\mathcal{L}}(\boldsymbol{\alpha}^*)\|_\infty \|\widehat{\Delta}\|_1 + 2\lambda_{\mathrm{lasso}} \{ \|\widehat{\Delta}_S\|_1 - \|\widehat{\Delta}_{S^c}\|_1 \} \\
    &\leq \lambda_{\mathrm{lasso}} \{ 3\|\widehat{\Delta}_S\|_1 - \|\widehat{\Delta}_{S^c}\|_1 \},
\end{align}
where last inequality shows that $\widehat{\Delta} \in \mathbb{C}_{3}(S)$.
\end{proof}
\begin{lemma}
 $||\widehat{\Delta}||_1 \leq 2 \sqrt{s}||\widehat{\Delta}||_2$ where $s = |S|$.\end{lemma}
 \begin{proof}
Since $S$ is the support of $\mb \alpha^*$
\begin{align}
    &||\boldsymbol{\alpha}_S^*||_1 = ||\boldsymbol{\alpha}^*||_1 \geq ||\boldsymbol{\alpha}^* + \widehat{\Delta}||_1 
    = ||\boldsymbol{\alpha}_S^* + \widehat{\Delta}_S||_1 + ||\widehat{\Delta}_{S^c}||_1 
    \geq ||\boldsymbol{\alpha}^*_S||_1 - ||\widehat{\Delta}_S||_1 + ||\widehat{\Delta}_{S^c}||_1.
\end{align}
where we used the fact that $\boldsymbol{\alpha}_{S^c}^* = 0 $ and triangle inequality. This implies that $\widehat{\Delta} \in \mathbb{C}_1(S)$. Therefore
\begin{equation}
    ||\widehat{\Delta}||_1 = ||\widehat{\Delta}_S||_1 + ||\widehat{\Delta}_{S^c}||_1 \leq 2 ||\widehat{\Delta}_S||_1  \leq 2\sqrt{s} ||\widehat{\Delta}||_2
\end{equation}
\end{proof}
With all above, we can then apply the RE condition.  Finally we have $||\widehat{\Delta}||_2^2 \leq \frac{3}{\kappa}\lambda_{\mathrm{lasso}} \sqrt{s} ||\widehat{\Delta}||_2$, which implies that 
\begin{equation}
    ||\hat{\boldsymbol{\alpha}} - \boldsymbol{\alpha}^*||_2 \leq \frac{3}{\kappa}\lambda_{\mathrm{lasso}} \sqrt{s}
\end{equation}

\subsection{Proof of Theorem \ref{probineqconst}}\label{pf:probineqconst}

We have from definition that $\nabla_{\boldsymbol{\alpha}}\hat{\mathcal{L}}(\boldsymbol{\alpha}) = \frac{1}{n} \sum_{j=1}^{n} \nabla_{\boldsymbol{\alpha}}\boldsymbol{m_{\alpha}}(t_j,\boldsymbol{x}_j)
$. %

\begin{lemma}
    Under the condition that the r.v. $\boldsymbol{f}(x) $ is bounded in $\ell_\infty$ norm, let $\sigma^2 = \max_{1\leq i\leq k} \boldsymbol{\Sigma}_{ii}$, where $\boldsymbol{\Sigma}$ is the covariance matrix of the random variable $\nabla_{\boldsymbol{\alpha}} \boldsymbol{m_{{\alpha}^*}}(t,\boldsymbol{x})$, the elements of $\nabla_{\boldsymbol{\alpha}} \hat{\mathcal{L}} (\boldsymbol{\alpha}^*)$ follow a zero-mean sub-Gaussian distribution with parameter $\sigma^2/n$.
\end{lemma}
\begin{proof}
\textbf{(1) Proof the sub-Gaussian: }
let $|g(t)|\leq C_g \in \mathbb{R}_{+}, |\partial_tg(t)|\leq C_{\partial_t g(t)} \in \mathbb{R}_{+}, ||\mathbb{E}_{q_t}[\boldsymbol{f}(x)]||_\infty \leq C_{E} \in \mathbb{R}_{+}$ and $||\boldsymbol{f}(x)||_\infty \leq C_f \in \mathbb{R}_{+}$ for $t \in [0,1]$. By using triangle inequality in $(i),(iii)$ and Cauchy-Schwarz inequality in $(ii)$ , we have by definition of $\nabla_{\boldsymbol{\alpha}}\boldsymbol{m_{\alpha}}(t,\boldsymbol{x})$
\begin{align}
    \left\|\nabla_{\boldsymbol{\alpha}}\boldsymbol{m_{\alpha}}(t,\boldsymbol{x})\right\|_\infty \leq \left\| \boldsymbol{\alpha}^* \right\|_\infty \left( C_g (C_f + C_E)^2 + C_f C_{\partial_t g(t)} \right) 
\end{align}
therefore by $M = ||\boldsymbol{\alpha}^*||(C_g(C_f+C_E)^2 + C_fC_{\partial_t g(t)})$ for simplicity
\begin{align}
    \left\| \nabla_{\boldsymbol{\alpha}} \hat{\mathcal{L}} \left( \boldsymbol{\alpha}^* \right)\right\|_\infty = \left\|\frac{1}{n} \sum_{j=1}^{n} \nabla_{\boldsymbol{\alpha}}\boldsymbol{m_{\alpha}}(t_j,\boldsymbol{x}_j)\right\|_\infty \leq M
\end{align}
which implies that  $\nabla_{\boldsymbol{\alpha}}\hat{\mathcal{L}}(\boldsymbol{\alpha}^*)$ is element-wise bounded hence all elements are sub-Gaussian by \citep{hoeffding1994probability}. 
Therefore, fixed number of data points $n$, then each element of $\nabla_{\boldsymbol{\alpha}}\hat{\mathcal{L}}(\boldsymbol{\alpha}^*)$ follows a sub-Gaussian distribution with parameter $\sigma^2/n$ by addition rule of variance.
\\
\textbf{(2) Proof of zero-mean: }We also assert that $\nabla_{\boldsymbol{\alpha}}\hat{\mathcal{L}}(\boldsymbol{\alpha}^*) \in \mathbb{R}^k$ is a zero-mean random variable. Recall that the conditional density function $q(x|t)$ as $q_t(x)$ where $q_t:\mathbb{R}^k\rightarrow \mathbb{R}$ and $q: \mathbb{R}\rightarrow\mathbb{R}$ represent the density of $t$, where $t$ is uniformly distributed in the domain $[0,1]$, we then have the following:
\begin{align}
    &\mathbb{E}_{x\sim q_t(x),\,t\sim q(t)}\left[\nabla_{\boldsymbol{\alpha}}\hat{\mathcal{L}}(\boldsymbol{\alpha}^*)\right] \notag &\\
    &= \frac{1}{n}\mathbb{E}_{x\sim q_t(x),\, t\sim q(t)}\left[\boldsymbol{\alpha}^*{\tilde{\boldsymbol{F}}}^\top \boldsymbol{G}{\tilde{\boldsymbol{F}}} + \boldsymbol{1}_n^\top \partial_t \boldsymbol{G} \boldsymbol{F}\right] \label{eq:1} &\\
    &\stackrel{(i)}{=} \mathbb{E}_{x\sim q_t(x),\, t\sim q(t)}\left[\boldsymbol{{\alpha}^*}{\left(\boldsymbol{f}(x)-\mathbb{E}_{q_t}[\boldsymbol{f}(y)]\right)} g(t){\left(\boldsymbol{f}(x)-\mathbb{E}_{q_t}[\boldsymbol{f}(y)]\right)} + (\partial_t g(t)) \boldsymbol{f}(x)\right] \label{eq:2} &\\
    &= \int_t \int_x q(t) q_t(x)\{
    \boldsymbol{{\alpha}^*}{(\boldsymbol{f}(x)-\mathbb{E}_{q_t}[\boldsymbol{f}(y)])}g(t)\left(\boldsymbol{f}(x)-\mathbb{E}_{q_t}[\boldsymbol{f}(y)]\right)
    + (\partial_t g(t)) \boldsymbol{f}(x)
    \} \mathrm{d}x \, \mathrm{d}t\\
    &\stackrel{(ii)}{=} \int_t \int_x q_t(x)\{
    (\partial_t\log q_t(x)) g(t)\left(\boldsymbol{f}(x)-\mathbb{E}_{q_t}[\boldsymbol{f}(y)]\right)
    + (\partial_t g(t)) \boldsymbol{f}(x)
    \} \mathrm{d}x \, \mathrm{d}t\\
    &= \int_t \int_x 
    (\partial_t q_t(x)) g(t) \left(\boldsymbol{f}(x)-\mathbb{E}_{q_t}[\boldsymbol{f}(y)]\right)
    + q_t(x) (\partial_t g(t)) \boldsymbol{f}(x)
    \} \mathrm{d}x \, \mathrm{d}t\\
    &= [q_t(x) g(t)\left(\boldsymbol{f}(x)-\mathbb{E}_{q_t}[\boldsymbol{f}(y)]\right)]^{t=1}_{t=0} 
    -\int_t \int_x q_t(x)
    \partial_t \left[ g(t) \left( \boldsymbol{f}(x)-\mathbb{E}_{q_t}[\boldsymbol{f}(y)]\right)\right]
    - q_t(x) (\partial_t g(t))\boldsymbol{f}(x) \mathrm{d}x \, \mathrm{d}t\\
    &\stackrel{(iii)}{=} \int_t \int_x -q_t(x)
    \left[ (\partial_t g(t)) \left(\boldsymbol{f}(x)-\mathbb{E}_{q_t}[\boldsymbol{f}(y)]\right) - g(t) \partial_t \mathbb{E}_{q_t}\left[ \boldsymbol{f}(y)\right] \right]
    + q_t(x)  (\partial_t g(t))\boldsymbol{f}(x) \mathrm{d}x \, \mathrm{d}t\\ 
    &= \int_t \int_x q_t(x)
    \left( [\partial_t g(t)]  \mathbb{E}_{q_t}[\boldsymbol{f}(y)] + g(t) \partial_t \mathbb{E}_{q_t}[\boldsymbol{f}(y)] \right)
     \mathrm{d}x \, \mathrm{d}t = \int_t \int_x q_t(x) \partial_t [g(t) \mathbb{E}_{q_t}[\boldsymbol{f}(y)]] \mathrm{d}x \, \mathrm{d}t\\
    &= [q_t(x)g(t)\mathbb{E}_{q_t}[\boldsymbol{f}(y)]]^{t=1}_{t=0} -\int_t \int_x (\partial_t q_t(x)) g(t) \mathbb{E}_{q_t}[\boldsymbol{f}(y)] \mathrm{d}x \, \mathrm{d}t\\     
    &\stackrel{(iv)}{=} -\int_t \partial_t \left(\int_x  q_t(x) \mathrm{d}x \,\right) g(t) \mathbb{E}_{q_t}[\boldsymbol{f}(y)]  \mathrm{d}t = \boldsymbol{0}
\end{align}
where $(i)$ follows that $\{f(x_1),..,f(x_n)\}$ are i.i.d; $(ii)$ follows from the definition of $\partial_t \log q_t(x)$ and $q(t) = 1$; $(iii),(iv)$ uses the boundary condition of $g(t)$. Therefore, each element of $\nabla_{\boldsymbol{\alpha}}\hat{\mathcal{L}}(\boldsymbol{\alpha}^*)$ is a zero-mean sub-Gaussian r.v. 

\end{proof}
\begin{lemma}\label{lemma:Jexp}
Let $J_i$ denotes i-th element of  $\nabla_{\boldsymbol{\alpha}}\hat{\mathcal{L}}(\boldsymbol{\alpha}^*)$, we have 
\begin{equation}
    \mathbb{E}[\max_{i\in [k]} |J_i|] \leq  2\sqrt{\frac{\sigma^2 \log k}{n}}
\end{equation}
\end{lemma}
\begin{proof}
 We have \(\{J_i\}_{i=1}^k\) is a sequence of zero-mean random variables, each follows sub-Gaussian with parameter \(\sigma^2 /n\). For any \(\lambda > 0\), we can use the convexity of the exponential function to obtain
\[
\exp\{\lambda \mathbb{E}[\max_{i \in [k]} J_i]\} \leq \mathbb{E}[\exp\{\lambda \max_{i \in [k]} J_i\}]
\]
by Jensen's inequality. And by the monotonicity of the exponential
\[
\mathbb{E}[\exp\{\lambda \max_{i \in [k]} J_i\}] = \mathbb{E}[\max_{i \in [k]} e^{\lambda J_i}] \leq \sum_{i=1}^{k} \mathbb{E}[e^{\lambda J_i}] \leq k e^{\frac{\lambda^2 \sigma^2}{2n}}.
\]
where last inequality follows the definition of sub-Gaussian r.v. Therefore,
\[
\mathbb{E}[\max_{i \in [k]} J_i] \leq \frac{\log k}{\lambda} + \lambda \frac{\sigma^2}{2n}
\]
and \(\lambda = \sqrt{\frac{2n \log k}{\sigma^2}}\) is optimal in $\lambda>0$; substituting we have 
\begin{equation}
\mathbb{E}\left[\max_{i \in [k]} J_i\right] \leq \frac{\sigma}{\sqrt{2n}} \sqrt{\log k} + \frac{\sigma}{\sqrt{2n}} \sqrt{\log k} = \sqrt{2 \sigma^2 \log k/n} \label{expconst}
\end{equation}
Since \ref{expconst} does not assume independence between individual \(J_i\). The result follows by
\[
    \max_{i \in [k]} |J_i| = \max \{|J_1|, \ldots, |J_k|\} = \max \{J_1, \ldots, J_k, -J_1, \ldots, -J_k\}
\]
\end{proof}
Consequently, from standard sub-Gaussian tail bound and definition of $\ell_{\infty}$ norm, we have
\begin{equation}
    \mathbb{P}\left[ \left\|\nabla_{\boldsymbol{\alpha}} \hat{\mathcal{L}}(\boldsymbol{\alpha}^*) \right\|_{\infty} \geq \sigma\left(\sqrt{\frac{2\log k}{n}}+\delta \right) \right]\leq 2e^{-\frac{n\delta^2}{2}} \;\; \text{for all } \delta>0
\end{equation}
Hence if we set $\lambda_{\mathrm{lasso}}  = 2\sigma\left(\sqrt{\frac{2\log k}{n}}+\delta\right)$, then we have the probability that $\lambda_{\mathrm{lasso}} \geq 2||\nabla_{\boldsymbol{\alpha}}\hat{\mathcal{L}}(\boldsymbol{\alpha}^*)||_{\infty}$ is in rate $1-2\exp\{-\Theta(n\delta^2)\}$, which implies 
\begin{equation}
    ||\hat{\boldsymbol{\alpha}} - \boldsymbol{\alpha}^*||_2 \leq \frac{6}{\kappa}\sqrt{s}\sigma\left(\sqrt{\frac{2\log k}{n}}+\delta\right)
\end{equation}
with probability greater than $1-2\exp\{-\frac{n\delta^2}{2}\}$ for all $\delta>0$.

\section{Theoretical Results of Debiased lasso}

\subsection{Variance Estimator}
\begin{lemma}
    $|||\hat{\mb \Sigma}(\hat{\boldsymbol{\alpha}})-\hat{\mb \Sigma}({\boldsymbol{\alpha}}^*)|||_\infty \leq L||\hat{\boldsymbol{\alpha}}-\boldsymbol{\alpha}^*||_1$
where $L$ is a constant.
\end{lemma}
\begin{proof}
    Apply the fact that there exists $L_0$ such that $\|\nabla_{\ba} m(\hat{\ba}) - \nabla_{\ba} m(\ba^*)\| \leq L_0 \|\hat{\boldsymbol{\alpha}}-\boldsymbol{\alpha}^*\|$
    after computing the form of each $\hat{\mb \Sigma}_{k,k^\prime}(\hat{\boldsymbol{\alpha}})-\hat{\mb \Sigma}_{k,k^\prime}({\boldsymbol{\alpha}}^*)$.
\end{proof}
\begin{lemma}\label{lemma:var}
    On the event that 
    \[
    \left\|\hat{\boldsymbol{\alpha}} - \boldsymbol{\alpha}^* \right\|_1 \leq \delta_\alpha / 2, \quad
    \left\|\tilde{\boldsymbol{\omega}}_j - \boldsymbol{\omega}_j^* \right\|_1 \leq \delta_{\boldsymbol{\omega}}, \quad \text{and} \quad
    |||\hat{\mb \Sigma}({\alpha}^*) - \boldsymbol{\Sigma} |||_\infty \leq \delta_\sigma / 2,
    \]
    we have 
    \begin{equation}
        \left|\hat{\sigma}^2_j - \sigma_j^2 \right| \leq (L \delta_\alpha + \delta_\sigma)\left( \left\|\boldsymbol{\omega}_j^* \right\|^2_1 + \delta^2_{\boldsymbol{\omega}} \right) + \delta_{\boldsymbol{\omega}}^2 |||\boldsymbol{\Sigma}|||_\infty
    \end{equation}
\end{lemma}

\begin{proof}
We have by definition of $\hat{\sigma}^2_j$
\begin{align}
    &\hat{\sigma}^2_j - \sigma_j^2 = \tilde{\boldsymbol{\omega}}_j^\top \hat{\mb \Sigma}(\hat{\boldsymbol{\alpha}}) \tilde{\boldsymbol{\omega}}_j - {\boldsymbol{\omega}_j^*}^\top \Sigma \boldsymbol{\omega}_j^* \\
    &\implies \left|\tilde{\boldsymbol{\omega}}_j^\top \hat{\mb \Sigma}(\hat{\boldsymbol{\alpha}}) \tilde{\boldsymbol{\omega}}_j - {\boldsymbol{\omega}_j^*}^\top \Sigma \boldsymbol{\omega}_j^*\right| \\
    &\leq \left|\tilde{\boldsymbol{\omega}}_j^\top \left(\hat{\mb \Sigma}(\hat{\boldsymbol{\alpha}}) - \Sigma\right) \tilde{\boldsymbol{\omega}}_j \right| + \left|(\tilde{\boldsymbol{\omega}}_j - \boldsymbol{\omega}_j^*)^\top \Sigma (\tilde{\boldsymbol{\omega}}_j - \boldsymbol{\omega}_j^*)\right| \\
    &\leq |||\hat{\mb \Sigma}(\hat{\boldsymbol{\alpha}}) - \Sigma|||_\infty \left\|\tilde{\boldsymbol{\omega}}_j\right\|^2_1 + |||\Sigma|||_\infty \left\|\tilde{\boldsymbol{\omega}}_j - \boldsymbol{\omega}_j^*\right\|^2_1 \\
    &\leq \left(|||\hat{\mb \Sigma}(\hat{\boldsymbol{\alpha}}) - \hat{\mb \Sigma}(\boldsymbol{\alpha}^*)|||_\infty + |||\hat{\mb \Sigma}(\boldsymbol{\alpha}^*) - \Sigma|||_\infty \right) \left\|\tilde{\boldsymbol{\omega}}_j\right\|^2_1 + |||\boldsymbol{\Sigma}|||_\infty \left\|\tilde{\boldsymbol{\omega}}_j - \boldsymbol{\omega}_j^*\right\|^2_1 \\
    &\leq \left(L\left\|\hat{\boldsymbol{\alpha}} - \boldsymbol{\alpha}^*\right\|_1 + |||\hat{\mb \Sigma}(\boldsymbol{\alpha}^*) - \boldsymbol{\Sigma}|||_\infty \right) \left\|\tilde{\boldsymbol{\omega}}_j\right\|^2_1 + |||\boldsymbol{\Sigma}|||_\infty \left\|\tilde{\boldsymbol{\omega}}_j - \boldsymbol{\omega}_j^*\right\|^2_1 \\
    &\leq (L\delta_{\boldsymbol{\alpha}} + \delta_\sigma) \left(\left\|\boldsymbol{\omega}_j^*\right\|^2_1 + \delta_{\boldsymbol{\omega}}^2 \right) + |||\boldsymbol{\Sigma}|||_\infty \delta_{\boldsymbol{\omega}}^2
\end{align}

as desired.
\end{proof}

\begin{lemma}\label{lemma:cov}
    There exists constants $c_0,c,c^\prime$ depend only on $M$ such that for any $t\in[c_0\sqrt{\frac{\log k}{n}},1]$ such that
    \begin{equation}
        \mathbb{P}\left(\|\hat{\mb \Sigma}(\boldsymbol{\alpha}^*) - \boldsymbol{\Sigma}\|_{\infty}\geq t \right)\leq c\exp\{-c^\prime t^2 n\}
    \end{equation}
\end{lemma}
\begin{proof}
    We have by denoting sample mean as $\hat{\boldsymbol{\mu}}$ and true mean as $\boldsymbol{\mu}$, we have for any $j,j^\prime\in [k]$
    \begin{align}
        &\hat{\mb \Sigma}_{j,j^\prime}(\boldsymbol{\alpha}^*) - \boldsymbol{\Sigma}_{j,j^\prime}= \\
        &\left(\frac{1}{n}\sum_{i=1}^n\left( [\nabla_{\boldsymbol{\alpha}}\boldsymbol{m_{\alpha}}(\boldsymbol{x}_i,t_i)]_j - \boldsymbol{\mu}_j\right)\left( [\nabla_{\boldsymbol{\alpha}}\boldsymbol{m_{\alpha}}(\boldsymbol{x}_i,t_i)]_{j^\prime} - \boldsymbol{\mu}_{j^\prime}\right) - \boldsymbol{\Sigma}_{j,j^\prime}\right) - \left(\boldsymbol{\mu}_j - \hat{\boldsymbol{\mu}}_j\right)\left(\boldsymbol{\mu}_{j^\prime} - \hat{\boldsymbol{\mu}}_{j^\prime}\right)
    \end{align}
Suppose $t$ satisfy the condition stated in lemma, and suppose 
    \begin{align}
        \left| \frac{1}{n}\sum_{i=1}^n \left([\nabla_{\boldsymbol{\alpha}}\boldsymbol{m_{\alpha}}(\boldsymbol{x}_i,t_i)]_j - \boldsymbol{\mu}_j\right)\left( [\nabla_{\boldsymbol{\alpha}}\boldsymbol{m_{\alpha}}(\boldsymbol{x}_i,t_i)]_{j^\prime} - \boldsymbol{\mu}_{j^\prime}\right)- \boldsymbol{\Sigma}_{j,j^\prime}\right| \leq t \text{            } \forall j,j^\prime \\
        \left|\hat{\boldsymbol{\mu}}_j  - \boldsymbol{\mu}_j\right| \leq t \text{            } \forall j
    \end{align}
On this event, 
    \begin{align}
        \|\hat{\mb \Sigma}(\boldsymbol{\alpha}^*) - \boldsymbol{\Sigma}\|_{\infty} &= \max_{j,j^\prime} \left|\hat{\mb \Sigma}_{j,j^\prime}(\boldsymbol{\alpha}^*) - \boldsymbol{\Sigma}_{j,j^\prime}\right|\nonumber\\
        &\leq \max_{j,j^\prime} \left| \frac{1}{n}\sum_{i=1}^n \left([\nabla_{\boldsymbol{\alpha}}\boldsymbol{m_{\alpha}}(\boldsymbol{x}_i,t_i)]_j - \boldsymbol{\mu}_j\right)\left( [\nabla_{\boldsymbol{\alpha}}\boldsymbol{m_{\alpha}}(\boldsymbol{x}_i,t_i)]_{j^\prime} - \boldsymbol{\mu}_{j^\prime}\right) - \boldsymbol{\Sigma}_{j,j^\prime} \right| + \max_{j} \left|\hat{\boldsymbol{\mu}}_j  - \boldsymbol{\mu}_j\right|^2\nonumber\\
        &\leq t+t^2 \leq 2t\nonumber
    \end{align}
By above statement, we have by boundness of random variable $\nabla_{\boldsymbol{\alpha}}\boldsymbol{m_{\alpha}}(\boldsymbol{x},t)$ and Hoffeding's inequality, there exists $c_1,c_2$ depend on $M$ only that 
\begin{align}
    \mathbb{P}\left(\left|\left([\nabla_{\boldsymbol{\alpha}}\boldsymbol{m_{\alpha}}(\boldsymbol{x}_i,t_i)]_j - \boldsymbol{\mu}_j\right)\left( [\nabla_{\boldsymbol{\alpha}}\boldsymbol{m_{\alpha}}(\boldsymbol{x}_i,t_i)]_{j^\prime} - \boldsymbol{\mu}_{j^\prime}\right)-\boldsymbol{\Sigma}_{j,j^\prime}\right| \geq t\right) &\leq 2\exp\{-c_1 t^2 n\}\\
    \mathbb{P}\left(\left|\hat{\boldsymbol{\mu}}_j  - \boldsymbol{\mu}_j\right|\geq t \right) &\leq 2\exp\{-c_2 t^2 n\}
\end{align}
Thus
\begin{equation}
   \mathbb{P}\left(\|\hat{\mb \Sigma}(\boldsymbol{\alpha}^*) - \boldsymbol{\Sigma}\|_{\infty}\geq t \right) \leq 2k\exp\{-c_1 t^2 n\} + 2k^2\exp\{-c_2 t^2 n\} \leq 4k^2 \exp\{-c_3 t^2 n\}\label{eq:88}
\end{equation}
for some $c_3$ depend on $M$ only. Finally we have
\begin{equation}
    \mathbb{P}\left(\|\hat{\mb \Sigma}(\boldsymbol{\alpha}^*) - \boldsymbol{\Sigma}\|_{\infty}\geq t \right)\leq c\exp\{-c^\prime t^2 n\}
\end{equation}
by simplifying equation (\ref{eq:88}) with bound of $t$ via choosing proper $c_0$ satisfying $c_0 > \sqrt{\frac{2}{c_3}}$.
\end{proof}

\subsection{Inverse Hessian Approximation}\label{app:consiinvhes}
\begin{lemma}[Consistency of Inverse Hessian estimator]\label{lemma:invhen}
Let $S_{{\boldsymbol{\omega}},j}$ be the support of $\boldsymbol{\omega}_j^*$, and $s_{{\boldsymbol{\omega}},j}= |S_{{\boldsymbol{\omega}},j}|$, under condition that
\begin{equation}
    \lambda_j \geq 2\|\nabla^2_{\boldsymbol{\alpha}}\hat{\mathcal{L}}_{\boldsymbol{\alpha}}(\boldsymbol{\alpha}^*)\boldsymbol{\omega}^*_j -\boldsymbol{e}_j\|_\infty
\end{equation}
we have
\begin{equation}
    \|\tilde{\boldsymbol{\omega}}_j - \boldsymbol{\omega}^*_j\|_1 \leq \frac{6}{\kappa_j}s_{{\boldsymbol{\omega}},j}\lambda_j
\end{equation}

\end{lemma}
\begin{proof}
    The proof is exactly the same as proof of \cref{consi}, the only difference is we replace $\nabla_{\boldsymbol{\alpha}}\hat{\mathcal{L}}(\boldsymbol{\alpha})$ with $\nabla_{\boldsymbol{\alpha}}\hat{\mathcal{L}}_{\boldsymbol{\alpha}}(\boldsymbol{\alpha}^*)\boldsymbol{\omega}^*_j -\boldsymbol{e}_j$ and $S$ with $S_{{\boldsymbol{\omega}},j}$. So we omit the proof.
\end{proof}

\begin{lemma}
There exists constants $c_0, c,c^\prime$ depend only on $M\|\boldsymbol{\omega}_j^*\|+1$ such that for any $t\in [c_0\sqrt{\frac{\log k}{n}},1]$, we have
\begin{equation}
      \mathbb{P}\left(\|\nabla^2_{\boldsymbol{\alpha}}\hat{\mathcal{L}}_{\boldsymbol{\alpha}}(\boldsymbol{\alpha}^*)\boldsymbol{\omega}^*_j -\boldsymbol{e}_j\|_{\infty}\geq t \right) \leq c\exp\{-c^\prime t^2 n \}
\end{equation}
\label{lemma:F5}
\end{lemma}
\begin{proof}
    The result follows from fact that the random variable $\nabla^2_{\boldsymbol{\alpha}}\hat{\mathcal{L}}_{\boldsymbol{\alpha}}(\boldsymbol{\alpha}^*)\boldsymbol{\omega}^*_j -\boldsymbol{e}_j$ follows sub-Gaussian distribution by definition of $\bome^*$. Therefore there exists $c_0$ such that
    \begin{equation}\label{eq:11}
        \mathbb{E}[\|\nabla^2_{\boldsymbol{\alpha}}\hat{\mathcal{L}}_{\boldsymbol{\alpha}}(\boldsymbol{\alpha}^*)\boldsymbol{\omega}^*_j -\boldsymbol{e}_j\|_{\infty}] \leq c_0\sqrt{\frac{\log k}{n}}
    \end{equation}
    The proof of equation (\ref{eq:11}) is the same as \cref{lemma:Jexp}. Finally we apply Hoffeding's inequality and obtain probabilistic inequality as desired.
\end{proof}

\subsection{Gaussian Approximation Bound}\label{pf:main}
This subsection presents the proof of Gaussian Approximation Bound(GAB). Lemma \ref{lemma:A} and \ref{lemma:combo} are useful lemmas for the proof. Theorem \ref{thm:remain} talks about GAB.
\begin{lemma}\label{lemma:A}
    For $\bome \in \mathbb{R}^{k}$, let
    \begin{equation}
        A_n(\bome) = \langle \bome , \nabla_{\boldsymbol{\alpha}}\hat{\mathcal{L}}(\boldsymbol{\alpha}^*)\rangle
    \end{equation}
    and
    \begin{equation}
        \sigma^2_n = \sigma_n^2(\bome) = Var[\sqrt{n}A_n(\bome)]
    \end{equation}
    Then
    \begin{equation}
        \sup_{z\in\mathbb{R}}|\mathbb{P}\{\sqrt{n}A_n(\bome)/\sigma_n \leq z\} - \Phi(z)| \leq \frac{2CM||\bome||}{\sqrt{n}\sigma_n}
    \end{equation}
    where $C = 3.3$ is a known constant.
\end{lemma}

\begin{proof}
    We have 
\begin{equation}
    \frac{\sqrt{n} A_n(\bome)}{\sigma_n} = \frac{1}{\sqrt{n}} \left\{ \sum_{i=1}^n \frac{\langle \bome, \nabla_{\boldsymbol{\alpha}} \boldsymbol{m_{\alpha^*}}(t_i,\boldsymbol{x}_i) \rangle}{\sigma_n} \right\}
\end{equation}
    and 
\begin{equation}
    \langle \bome, \nabla_{\boldsymbol{\alpha}} \boldsymbol{m_{\alpha^*}}(t_i,\boldsymbol{x}_i)  \rangle \leq \frac{2M \|\bome\|}{\sigma_n}
\end{equation}
for all $i$. Finally the Berry-Esseen inequality (Theorem 3.4 of \citep{chen2010normal}) yields
\begin{equation}
        \sup_{z\in\mathbb{R}}|\mathbb{P}\{\sqrt{n}A_n(\bome)/\sigma_n \leq z\} - \Phi(z)| \leq \frac{2CM||\bome||}{\sqrt{n}\sigma_n}
\end{equation}
where $C = 3.3$ is a known constant.
\end{proof}

\begin{lemma}(Lemma D.3 of \citep{barber2018rocket})\label{lemma:combo}
    If 
    \begin{equation}
        \sup_{z\in\mathbb{R}}|\mathbb{P}\{A\leq z\} -\Phi(z) | \leq \epsilon_A \text{ and } \mathbb{P}\{|B|\leq \delta_B, |C|\leq \delta_C \} \geq 1-\epsilon_{BC}
    \end{equation}
    for some $\epsilon_A,\epsilon_{BC},\delta_B,\delta_C \in [0,1)$, then 
    \begin{equation}
        \sup_{z\in\mathbb{R}}|\mathbb{P}\{(A+B)/(1+C)\leq z\}-\Phi(z) |\leq \delta_B + \frac{\delta_C}{1-\delta_C} + \epsilon_A + \epsilon_{BC}
    \end{equation}
\end{lemma}

\begin{theorem}{(Restatement of \cref{thm:main})}\label{thm:remain}
    For $\delta_{\boldsymbol{\alpha}}, \delta_{\boldsymbol{\omega}}, \lambda_{\mathrm{lasso}}, \lambda_j, \delta_{\sigma} \in [0,1)$, define the event
    \begin{equation}
        \boldsymbol{\mathcal{E}} =
        \left\{
        \begin{array}{ll}
        2\left\|\nabla_{\boldsymbol{\alpha}}\hat{\mathcal{L}}(\boldsymbol{\alpha}^*) \right\|_\infty \leq \lambda_{\mathrm{lasso}}, &
        2\left\|\nabla^2_{\boldsymbol{\alpha}}\hat{\mathcal{L}}(\boldsymbol{\alpha}^*)  \boldsymbol{\omega}^*_j - \boldsymbol{e}_j\right\|_\infty \leq \lambda_j,  
        \\
        \left\|\hat{\boldsymbol{\alpha}} - \boldsymbol{\alpha}^* \right\|_1 \leq \delta_{\boldsymbol{\alpha}}, &
        \left\|\tilde{\boldsymbol{\omega}}_j - \boldsymbol{\omega}_j^* \right\|_1 \leq \delta_{\boldsymbol{\omega}},
        \\
        ||\hat{\mb \Sigma}(\boldsymbol{\alpha}^*) - \boldsymbol{\Sigma}||_\infty \leq \delta_\sigma / 2
        \end{array}
        \right\}
    \end{equation}
    Suppose $\mathbb{P}(\boldsymbol{\mathcal{E}}) \geq 1 - \epsilon$. We have 
    \begin{equation}
        \sup_{z \in \mathbb{R}} \left|\mathbb{P}\left\{\frac{\sqrt{n}(\tilde{\alpha}_j - \alpha_j^*)}{\hat{\sigma}_j} \leq z \right\} - \Phi(z)\right| \leq \Delta_1 + \Delta_2 + \Delta_3 + \epsilon
    \end{equation}
\end{theorem}

\begin{proof}
We have by definition of debiased lasso
\begin{align}
     &\tilde{{\alpha}}_j = \hat{{\alpha
     }}_j - \tilde{\bome}_j^\top \nabla_{\boldsymbol{\alpha}}\hat{\mathcal{L}}(\boldsymbol{\hat{\boldsymbol{\alpha}}})
        \\
    &= \hat{{\alpha}}_j - {\bome^*_j}^\top \nabla_{\boldsymbol{\alpha}}\hat{\mathcal{L}}(\hat{\boldsymbol{\alpha}}) + (\tilde{\bome}_j - \bome_j^*)^\top\nabla_{\boldsymbol{\alpha}}\hat{\mathcal{L}}(\hat{\boldsymbol{\alpha}})
        \\
    & = \hat{{\alpha}}_j - {\bome^*_j}^\top\left(\nabla_{\boldsymbol{\alpha}}\hat{\mathcal{L}}(\boldsymbol{\alpha}^*)+\nabla^2_{\boldsymbol{\alpha}}\hat{\mathcal{L}}(\boldsymbol{\alpha}^*)(\hat{\boldsymbol{\alpha}} - \boldsymbol{\alpha}^*)\right) + (\tilde{\bome}_j - \bome_j^*)\nabla_{\boldsymbol{\alpha}}\hat{\mathcal{L}}(\hat{\boldsymbol{\alpha}})
        \\
    &\implies \tilde{{\alpha}}_j - {\ba}_j^* = \underbrace{-{\bome_j^*}^\top\nabla_{\boldsymbol{\alpha}}\hat{\mathcal{L}}(\ba^*)}_{A} -\underbrace{\left(\nabla^2_{\boldsymbol{\alpha}}\hat{\mathcal{L}}(\boldsymbol{\alpha}^*){\bome_j^*}-\boldsymbol{e}_j\right)(\hat{\boldsymbol{\alpha}} - \boldsymbol{\alpha}^*) + (\tilde{\bome}_j - \bome_j^*)^\top\nabla_{\boldsymbol{\alpha}}\hat{\mathcal{L}}(\hat{\boldsymbol{\alpha}})}_{B}
\end{align}
Therefore we have by lemma \ref{lemma:A} that
\begin{equation}
    \sup_{z\in\mathbb{R}}|\mathbb{P}\{\sqrt{n}A/\sigma_n \leq z\} - \Phi(z)| \leq \frac{6.6M||\bome_j^*||}{\sqrt{n}\sigma_j} = \Delta_1
\end{equation}
        Furthermore we have 
\begin{align}
        &B = \left(\nabla^2_{\boldsymbol{\alpha}}\hat{\mathcal{L}}(\boldsymbol{\alpha}^*){\bome_j^*}-\boldsymbol{e}_j\right)(\hat{\boldsymbol{\alpha}} - \boldsymbol{\alpha}^*) + (\tilde{\bome}_j - \bome_j^*)^\top\nabla_{\boldsymbol{\alpha}}\hat{\mathcal{L}}(\hat{\boldsymbol{\alpha}}) \\
        & = \underbrace{\left(\nabla^2_{\boldsymbol{\alpha}}\hat{\mathcal{L}}(\boldsymbol{\alpha}^*){\bome_j^*}-\boldsymbol{e}_j\right)(\hat{\boldsymbol{\alpha}} - \boldsymbol{\alpha}^*)}_{B_1}\\
        &+ \underbrace{(\tilde{\bome}_j - \bome_j^*)^\top\nabla_{\boldsymbol{\alpha}}\hat{\mathcal{L}}(\ba^*)}_{B_2} + \underbrace{(\tilde{\bome}_j - \bome_j^*)^\top\left(\nabla_{\boldsymbol{\alpha}}\hat{\mathcal{L}}(\hat{\boldsymbol{\alpha}}) - \nabla_{\boldsymbol{\alpha}}\hat{\mathcal{L}}(\boldsymbol{\alpha}^*)\right)}_{B_{3}}
\end{align}
and by Taylor expansion
\begin{equation}
    \nabla_{\boldsymbol{\alpha}}\hat{\mathcal{L}}(\hat{\boldsymbol{\alpha}}) - \nabla_{\boldsymbol{\alpha}}\hat{\mathcal{L}}(\boldsymbol{\alpha}^*) = (\hat{\boldsymbol{\alpha}}-\boldsymbol{\alpha}^*)\frac{{\tilde{\boldsymbol{F}}}^\top G \tilde{\boldsymbol{F}}}{n}
\end{equation}
Hence by Holder's inequality:
\begin{align}
    &|B_1| \leq 2 \left\|\hat{\boldsymbol{\alpha}} - \boldsymbol{\alpha}^* \right\|_1 \left\|\nabla^2_{\boldsymbol{\alpha}} \hat{\mathcal{L}}(\boldsymbol{\alpha}^*) \boldsymbol{\omega}_j^* - \boldsymbol{e}_j \right\|_\infty \leq \lambda_j \delta_{\boldsymbol{\alpha}}
    \\
    &|B_2| \leq 2 \left\|\tilde{\boldsymbol{\omega}}_j - \boldsymbol{\omega}_j^* \right\|_1 \left\|\nabla_{\boldsymbol{\alpha}} \hat{\mathcal{L}}(\boldsymbol{\alpha}^*)  \right\|_\infty \leq \delta_{\boldsymbol{\omega}} \lambda_{\mathrm{lasso}}
    \\
    &|B_3| \leq K \left\|\hat{\boldsymbol{\alpha}} - \boldsymbol{\alpha}^* \right\|_\infty \left\|\tilde{\boldsymbol{\omega}}_j - \boldsymbol{\omega}_j^* \right\|_1 \leq 4K \left\|\hat{\boldsymbol{\alpha}} - \boldsymbol{\alpha}^* \right\|_1 \left\|\tilde{\boldsymbol{\omega}}_j - \boldsymbol{\omega}_j^* \right\|_1 \leq K \delta_{\boldsymbol{\alpha}} \delta_{\boldsymbol{\omega}}
\end{align}

where $K = C
_f^2C_g$ is a constant and $C_f,C_g$ are bounds of $||\boldsymbol{f(x)}||_\infty$ and $|g(t)|$ respectively, recall we assume bounded sufficient statistics. Therefore we have 
\begin{equation}
    \frac{\sqrt{n}|B|}{\sigma_j} \leq \frac{\sqrt{n}(\lambda_j \delta_{\boldsymbol{\alpha}} + \delta_{\boldsymbol{\omega}} \lambda_{\mathrm{lasso}} + K \delta_{\boldsymbol{\alpha}} \delta_{\boldsymbol{\omega}})}{\sigma_j} = \Delta_2
\end{equation}

Moreover, by \cref{lemma:var},
\begin{equation}
    C = \left|\frac{\hat{\sigma}_j}{\sigma_j} - 1\right| = \left|\frac{\hat{\sigma}_j - \sigma_j}{\sigma_j}\right| 
    \leq \left|\frac{\hat{\sigma}_j^2 - \sigma_j^2}{\sigma_j^2}\right| 
    \leq \frac{(2L \delta_{\boldsymbol{\alpha}} + \delta_\sigma)\left(\|\boldsymbol{\omega}_j^*\|^2 + \delta_{\boldsymbol{\omega}}^2\right) + |||\boldsymbol{\Sigma}|||_\infty \delta_{\boldsymbol{\omega}}^2}{\sigma_j^2} = \delta_C
\end{equation}

where first inequality can be derived using the difference of squares formula.
Finally we apply \ref{lemma:combo}, obtain
\begin{equation}
        \sup_{z\in\mathbb{R}}\left|\mathbb{P}\{\sqrt{n}(\tilde{\alpha}_j-\alpha_j^*)/\hat{\sigma}_j\leq z \} - \Phi(z)\right| \leq \Delta_1+\Delta_2+\Delta_3 +\epsilon
\end{equation}
where $\Delta_3 = \frac{\delta_C}{1-\delta_C}$.
\end{proof}

\subsection{Proof of Theorem \ref{thm:main2}}\label{pf:main2}

\begin{theorem}
    Denoting $s_{\bome,j}$ as the cardinality of support set of $\bome_j^*$ and $s$ as the cardinality of support set of $\ba^*$ respectively. Let $\tilde{\alpha}_j$ be debiased lasso estimator with tuning parameter
    \begin{equation}
        \lambda_{\mathrm{lasso}} \in \mathcal{O}\left(\sqrt{\frac{\log k}{n}}\right); \text{ and }  \lambda_j \in  \mathcal{O}\left(\sqrt{s_{\boldsymbol{\omega},j} \frac{\log k}{n}}\right)
        \end{equation}
    we have there exists positive constants $c,c^\prime$ such that
    \begin{equation}
    \sup_{z\in\mathbb{R}} \left|\mathbb{P}\{\sqrt{n}(\tilde{\alpha}_j-\alpha^*_j)/\hat{\sigma}_j \leq z\}-\Phi(z)\right|\leq \mathcal{O}\left(\sqrt{n}s^{3/2}_{\boldsymbol{\omega},j} s{\frac{\log k}{n}}\right) + c\exp\{-c^\prime \log k\}
    \end{equation}
        
\end{theorem}
\begin{proof}
Consider the event $\boldsymbol{\mathcal{E}}^L $
    \begin{align}
        2||\nabla_{\boldsymbol{\alpha}}\hat{\mathcal{L}}(\boldsymbol{\alpha}^*)||_{\infty} &\leq  \lambda_{\mathrm{lasso}} \label{eq:92}
        \\
        2||\nabla^2_{\boldsymbol{\alpha}}\hat{\mathcal{L}}(\boldsymbol{\alpha}^*)\bome^*_j-\boldsymbol{e}_j||_{\infty} &\leq  \lambda_{j}\label{eq:93}
        \\
        ||\hat{\mb \Sigma}(\boldsymbol{\alpha}^*)-\boldsymbol{\Sigma}||_{\infty}&\leq \sqrt{s}\lambda_{\mathrm{lasso}}
    \end{align}
We have $\boldsymbol{\mathcal{E}}^L \subseteq \boldsymbol{\mathcal{E}}$ under \cref{ass:re} since the following:\\
First we have by \cref{consi} that from equation (\ref{eq:92}) with \cref{ass:re}
\begin{equation}
    \|\hat{\boldsymbol{\alpha}}-\boldsymbol{\alpha}^*\|_1 \leq \frac{6}{\kappa}\lambda_{\mathrm{lasso}}s \in \mathcal{O}\left(s\sqrt{\frac{\log k}{n}}\right)
\end{equation}
In addition, we have from \cref{lemma:invhen} that from equation (\ref{eq:93})
\begin{equation}
    \|\tilde{\boldsymbol{\omega}}_j-\boldsymbol{\omega}^*_j\|_1 \leq \frac{6}{\kappa}\lambda_{j}s_{\boldsymbol{\omega},j}\in\mathcal{O}\left(\sqrt{s_{\boldsymbol{\omega},j}^3 \frac{\log k}{n}}  \right)\label{eq:delta2O}
\end{equation}
therefore $\boldsymbol{\mathcal{E}}^L \subseteq \boldsymbol{\mathcal{E}}$ and we have 
\begin{equation}
    \Delta_2 \in \mathcal{O}\left(\sqrt{n}s^{3/2}_{\boldsymbol{\omega},j} s{\frac{\log k}{n}}\right)
\end{equation}
We ignore $\Delta_1$ and $\Delta_3$ since they are of smaller order.\\
Next we bound $\mathbb{P}\left(\{\boldsymbol{\mathcal{E}^{L}}\}^{c}\right)$. Let
\begin{align}
    \boldsymbol{\mathcal{E}}_1 &= \{2||\nabla_{\boldsymbol{\alpha}}\hat{\mathcal{L}}(\boldsymbol{\alpha}^*)||_{\infty} \leq \lambda_{\mathrm{lasso}}\}\\
    \boldsymbol{\mathcal{E}}_2 &= \{2||\nabla^2_{\boldsymbol{\alpha}}\hat{\mathcal{L}}(\boldsymbol{\alpha}^*)\bome^*_j-\boldsymbol{e}_j||_{\infty} \leq \lambda_{j}\}\\
    \boldsymbol{\mathcal{E}}_3 &= \{||\hat{\mb \Sigma}(\boldsymbol{\alpha}^*)-\boldsymbol{\Sigma}||_{\infty}\leq \sqrt{s}\lambda_{\mathrm{lasso}}\}
\end{align}
It is obvious that 
\begin{equation}
    \mathbb{P}\left(\boldsymbol{\{{\mathcal{E}}^{L}}\}^c\right)\leq \sum_{i=1}^3  \mathbb{P}\left(\boldsymbol{\mathcal{E}}^{c}_i\right)
\end{equation}
Under bounded condition, we have the following:
By \cref{lemma:Jexp} and \cref{lemma:F5}, there exist constants \(c_1\), \(c_1^\prime\), \(c_2\), and \(c_2^\prime\).
\begin{align}
    \mathbb{P}\left(\boldsymbol{\mathcal{E}}^{c}_1\right) &= \mathbb{P}\left(2||\nabla_{\boldsymbol{\alpha}}\hat{\mathcal{L}}(\boldsymbol{\alpha}^*)||_{\infty} > \lambda_{k}\right)\leq c_1\exp\{-c_1^\prime\log k\}\\
    \mathbb{P}\left(\boldsymbol{\mathcal{E}}^{c}_2\right) &= \mathbb{P}\left(2||\nabla^2_{\boldsymbol{\alpha}}\hat{\mathcal{L}}(\boldsymbol{\alpha}^*)-\boldsymbol{e}_k||_{\infty} > \lambda_{\mathrm{lasso}}\right)\leq c_2\exp\{-c_2^\prime\log k\}
\end{align}
and by \cref{lemma:cov} there exists $c_3,c_3^\prime$
\begin{equation}
     \mathbb{P}\left(\boldsymbol{\mathcal{E}}^{c}_3\right) = \mathbb{P}\left(||\hat{\mb \Sigma}(\boldsymbol{\alpha}^*)-\boldsymbol{\Sigma}||_{\infty}> \sqrt{s}\lambda_{\mathrm{lasso}}\right)\leq c_3\exp\{-c_3^\prime\log k\}\label{eq:107}
\end{equation}
Therefore there exists $c,c^\prime$
\begin{equation}\label{eq:108}
     \mathbb{P}\left(\boldsymbol{\{{\mathcal{E}}^{L}}\}^c\right) \leq c\exp\{-c^\prime \log k\}
\end{equation}
Finally we complete the proof with combining the bounds of (\ref{eq:108}) and (\ref{eq:delta2O}),
\begin{equation}
    \sup_{z\in\mathbb{R}} \left|\mathbb{P}\{\sqrt{n}(\tilde{\alpha}_j-\alpha^*_j)/\hat{\sigma}_j \leq z\}-\Phi(z)\right|\leq \mathcal{O}\left(\sqrt{n}s^{3/2}_{\boldsymbol{\omega},j} s{\frac{\log k}{n}}\right) + c\exp\{-c^\prime \log k\}
\end{equation}
\end{proof}

\section{Score Model, a Gaussian Example} \label{app:gaussian_score}

Let us group this example into three parts: fixed mean and time-dependent variance ($\mu$ and $\sigma_t^2$), time-dependent mean and fixed variance ($\mu_t$ and $\sigma^2$), and time-dependent mean and variance ($\mu_t$ and $\sigma_t^2$). Each one is detailed in distinct sections below.

Across all three cases we aim to verify the following equation holds 
\begin{equation}
\partial_t \log q_t = \langle \partialthetasmall, \mb{f}({x}) - \E_{q_t} [\mb{f}({y})] \rangle,
\label{eq:app:dlogqt}
\end{equation}
which is our formulation given by \cref{thm:score}. For the exponential family, the natural parameterisation of the Gaussian distribution is given by
\begin{equation}
\mb\theta = \begin{bmatrix}
    \frac{\mu}{\sigma^2} \\ -\frac{1}{2\sigma^2}
\end{bmatrix}
\label{eq:app:naturalparams}
\end{equation}
for a given $\mu$ and $\sigma^2$.

\paragraph{Fixed mean and time-dependent variance} 

Let the fixed mean be written as $\mu$ and the time-dependent variance be written as $\sigma_t^2$. 
Firstly, according to this Gaussian distribution, the LHS of \cref{eq:app:dlogqt} is given by
\begin{align}
\partial_t \log q_t &= \partial_t \br{\frac{-(x-\mu)^2}{2\sigma_t^2}} - \partial_t \log (\sqrt{2\pi}\sigma_t) \nonumber \\
&= - \partial_t \br{\frac{x^2}{2\sigma_t^2}} + \partial_t\br{\frac{x\mu}{\sigma_t}} - \partial_t\br{\frac{\mu^2}{2\sigma_t^2}} - \partial_t \br{\log \sigma_t} \nonumber \\
&= - x^2 \partial_t \br{\frac{1}{2\sigma_t^2}} + x\mu \partial_t\br{\frac{1}{\sigma_t}} - \mu^2 \partial_t\br{\frac{1}{2\sigma_t^2}} - \frac{\partial_t \sigma_t}{\sigma_t} \label{eq:app:fixed_mean_time_var}
\end{align}
We aim to show that when $\mb{f}(x) = [x, x^2]$, the RHS of \cref{eq:app:dlogqt} is equal to this. We first write
\[
\mb\theta_t = \begin{bmatrix}
    \frac{\mu}{\sigma_t^2} \\ -\frac{1}{2\sigma_t^2}
\end{bmatrix}, \qquad \partialthetasmall = \begin{bmatrix}
    \mu \partial_t \br{\frac{1}{\sigma_t^2}} \\ -\partial_t \br{\frac{1}{2\sigma_t^2}}
\end{bmatrix}.
\]
\begin{align*}
    \left\langle \partial_t \theta_t, \mb{f}({x}) - \E_{q_t} [\mb{f}({y})] \right\rangle &= \left\langle \begin{bmatrix}
        \partial_t \br{ \frac{\mu}{\sigma_t^2} } \\
        \partial_t \br{ -\frac{1}{2\sigma_t^2} }         
    \end{bmatrix}, \begin{bmatrix}
        x \\
        x^2 
    \end{bmatrix} - \begin{bmatrix}
        \Eqtsmall[y] \\
        \Eqtsmall[y^2] 
    \end{bmatrix} \right\rangle \\
    &= \left\langle \begin{bmatrix}
        \partial_t \br{ \frac{\mu}{\sigma_t^2} } \\
        \partial_t \br{ -\frac{1}{2\sigma_t^2} }         
    \end{bmatrix}, \begin{bmatrix}
        x \\
        x^2 
    \end{bmatrix} - \begin{bmatrix}
        \mu \\
        \sigma_t^2 + \mu^2 
    \end{bmatrix} \right\rangle \\
    &= \left\langle \begin{bmatrix}
        \partial_t \br{ \frac{\mu}{\sigma_t^2} } \\
        \partial_t \br{ -\frac{1}{2\sigma_t^2} }         
    \end{bmatrix}, \begin{bmatrix}
        x - \mu \\
        x^2 - \sigma_t^2 - \mu^2
    \end{bmatrix}\right\rangle \\
    &= \partial_t \br{ \frac{\mu}{\sigma_t^2} } (x - \mu) - \partial_t \br{\frac{1}{2\sigma_t^2} }(x^2 - \sigma_t^2 - \mu^2) \\
    &=x\mu \partial_t\br{ \frac{1}{\sigma_t^2} } - \mu^2 \partial_t\br{ \frac{1}{\sigma_t^2} } - x^2\partial_t\br{ \frac{1}{2\sigma_t^2} } + \sigma_t^2 \partial_t\br{ \frac{1}{2\sigma_t^2} } + \mu^2 \partial_t\br{ \frac{1}{2\sigma_t^2} } \\
    &=x\mu \partial_t\br{ \frac{1}{\sigma_t^2} } - \mu^2 \partial_t\br{ \frac{1}{2\sigma_t^2} } - x^2\partial_t\br{ \frac{1}{2\sigma_t^2} } + \sigma_t^2 \partial_t\br{ \frac{1}{2\sigma_t^2} }. 
\end{align*}
By the chain rule,
\[
\sigma_t^2\partial_t\br{\frac{1}{2\sigma_t^2}} = \frac{\sigma_t^2}{2}\partial_t\br{\sigma_t^{-2}} = -\frac{\sigma_t^2}{2} \frac{2}{\sigma_t^3}\partial_t \sigma_t = -\frac{\partial_t \sigma_t}{\sigma_t},
\]
which, substituted into the equation above, leaves
\[
x\mu \partial_t\br{ \frac{1}{\sigma_t^2} } - \mu^2 \partial_t\br{ \frac{1}{2\sigma_t^2} } - x^2\partial_t\br{ \frac{1}{2\sigma_t^2} } - \frac{\partial_t\sigma_t}{\sigma_t}
\]
for which all terms match the terms in \cref{eq:app:fixed_mean_time_var} as desired.

Since we are doing very similar operations across all examples, the following two derivations will be lighter on details but should be straightforward to follow.

\paragraph{Time-dependent mean and fixed variance}
Firstly, write $\mu_t$ and $\sigma^2$ as the mean and variance of this Gaussian distribution, respectively. The time score function, i.e. the LHS of \cref{eq:app:dlogqt} is given by
\begin{align}
    \partial_t \log q_t &= \partial_t \br{\frac{-(x-\mu_t)^2}{2\sigma^2}} - \partial_t \log (\sqrt{2\pi}\sigma) \nonumber \\
&= \frac{2x \partial_t \mu_t - \partial_t (\mu_t^2)}{2\sigma^2} .\label{eq:app:dlogqt_time_mean_fixed_var}
\end{align}
The natural parameters are given by
\[
\mb\theta_t = \begin{bmatrix}
    \frac{\mu_t}{\sigma^2} \\ -\frac{1}{2\sigma^2}
\end{bmatrix}, \qquad \partialthetasmall = \begin{bmatrix}
    \partial_t \br{\frac{\mu_t}{\sigma^2}} \\ \partial_t \br{-\frac{1}{2\sigma^2}}
\end{bmatrix} = \begin{bmatrix}
    \frac{\partial_t \mu_t}{\sigma^2} \\ 0
\end{bmatrix},
\]
and so we consider the first dimension only. The RHS of \cref{eq:app:dlogqt}, when $\mb{f}(x) = x$ is given by
\begin{align*}
    \left\langle \partial_t \theta_t, \mb{f}({x}) - \E_{q_t} [\mb{f}({y})] \right\rangle &= \brangle{ \frac{ \partial_t \mu_t}{\sigma^2}, x - \mu_t} \\
    &= \frac{x \partial_t \mu_t - \mu_t \partial_t \mu_t}{\sigma^2} \\
    &= \frac{x \partial_t \mu_t - \frac{\partial_t (\mu_t^2)}{2}}{\sigma^2} \\
    &= \frac{2x \partial_t \mu_t - \partial_t (\mu_t^2)}{2\sigma^2} 
\end{align*}
where the penultimate line is due to the chain rule. This matches \cref{eq:app:dlogqt_time_mean_fixed_var} as desired.

\paragraph{Time-dependent mean and variance}

Firstly, write $\mu_t$ and $\sigma_t^2$ as the mean and variance of this Gaussian distribution, respectively. The time score function, i.e. the LHS of \cref{eq:app:dlogqt} is given by
\begin{align}
    \partial_t \log q_t &= \partial_t \br{\frac{-(x-\mu_t)^2}{2\sigma_t^2}} - \partial_t \log (\sqrt{2\pi}\sigma_t) \nonumber \\
&= -x^2 \partial_t \br{\frac{1}{2\sigma_t^2}} + x\frac{\partial_t \mu_t}{\sigma_t^2} + x\mu_t \partial_t \br{\frac{1}{\sigma_t^2}} - \frac{\partial_t (\mu_t^2)}{2\sigma_t^2} - \mu_t^2\partial_t \br{\frac{1}{2\sigma_t^2}} - \frac{\partial_t \sigma_t}{\sigma_t}.
\label{eq:app:dlogqt_time_mean_time_var}
\end{align}
The natural parameters are given by
\[
\mb\theta_t = \begin{bmatrix}
    \frac{\mu_t}{\sigma_t^2} \\ -\frac{1}{2\sigma_t^2}
\end{bmatrix}, \qquad \partialthetasmall = \begin{bmatrix}
    \partial_t \br{\frac{\mu_t}{\sigma_t^2}} \\ \partial_t \br{-\frac{1}{2\sigma_t^2}}
\end{bmatrix} = \begin{bmatrix}
    \frac{\partial_t \mu_t}{\sigma_t^2} + \mu_t\partial_t \br{\frac{1}{\sigma_t^2}} \\ -\partial_t\br{\frac{1}{2\sigma_t^2}}
\end{bmatrix},
\]
The RHS of \cref{eq:app:dlogqt} when $\mb{f}(x) = [x, x^2]$ is given by
\begin{align*}
    \left\langle \partial_t \theta_t, \mb{f}({x}) - \E_{q_t} [\mb{f}({y})] \right\rangle &= \left\langle \begin{bmatrix}
        \frac{\partial_t \mu_t}{\sigma_t^2} + \mu_t\partial_t \br{\frac{1}{\sigma_t^2}} \\ -\partial_t\br{\frac{1}{2\sigma_t^2}}     
    \end{bmatrix}, \begin{bmatrix}
        x - \mu_t \\
        x^2 - \sigma_t^2 - \mu_t^2
    \end{bmatrix}\right\rangle \\
    &= \br{\frac{\partial_t \mu_t}{\sigma_t^2} + \mu_t\partial_t \br{\frac{1}{\sigma_t^2}}}( x - \mu_t) - \partial_t\br{\frac{1}{2\sigma_t^2}}( x^2 - \sigma_t^2 - \mu_t^2)\\
    &= x \frac{\partial_t \mu_t}{\sigma_t^2} + x\mu_t\partial_t \br{\frac{1}{\sigma_t^2}} - \mu_t\frac{\partial_t \mu_t}{\sigma_t^2}  - \mu_t^2\partial_t \br{\frac{1}{\sigma_t^2}} - x^2 \partial_t\br{\frac{1}{2\sigma_t^2}} + \sigma_t^2 \partial_t\br{\frac{1}{2\sigma_t^2}} + \mu_t^2 \partial_t\br{\frac{1}{2\sigma_t^2}} \\
    &= x \frac{\partial_t \mu_t}{\sigma_t^2} + x\mu_t\partial_t \br{\frac{1}{\sigma_t^2}} - \mu_t\frac{\partial_t \mu_t}{\sigma_t^2}  + \mu_t^2\partial_t \br{\frac{1}{2\sigma_t^2}} - x^2 \partial_t\br{\frac{1}{2\sigma_t^2}} + \sigma_t^2 \partial_t\br{\frac{1}{2\sigma_t^2}}\\
    &= x \frac{\partial_t \mu_t}{\sigma_t^2} + x\mu_t\partial_t \br{\frac{1}{\sigma_t^2}} - \mu_t\frac{\partial_t \mu_t}{\sigma_t^2}  + \mu_t^2\partial_t \br{\frac{1}{2\sigma_t^2}} - x^2 \partial_t\br{\frac{1}{2\sigma_t^2}} - \frac{\partial_t \sigma_t}{\sigma_t} \\
    &= x \frac{\partial_t \mu_t}{\sigma_t^2} + x\mu_t\partial_t \br{\frac{1}{\sigma_t^2}} - \frac{\partial_t (\mu_t^2)}{2\sigma_t^2}  - \mu_t^2\partial_t \br{\frac{1}{2\sigma_t^2}} - x^2 \partial_t\br{\frac{1}{2\sigma_t^2}} - \frac{\partial_t \sigma_t}{\sigma_t},
\end{align*}
where the last three equalities, in their respective order, are due to collecting like terms, the chain rule on the last term, and the chain rule again on the third term. This matches \cref{eq:app:dlogqt_time_mean_time_var}, completing the proof.

\section{Simulation Study Implementation Details}\label{app:simulation}

\subsection{Construction of Gaussian Graphical Models for Estimation}
\label{sec:spartsm}
In \cref{sec:diff.est}, we consider two Gaussian Graphical Models. 

\textbf{Random Sine Gaussian Graphical Models} refers to Gaussian Graphical Models whose edges that vary with time are sine functions of $t$ and the edges that depend on $t$ are randomly chosen with a Bernoulli distribution with probability 0.02. We set the diagonal element $2$ and off-diagonal elements are
\begin{equation}
    \Theta_{i,j}^\prime(t) = \Theta_{j,i}^\prime(t) = \left\{
    \begin{array}{l}
    0.5 \cdot \sin(10t) \text{ w.p. } 0.02 \\
    0 \text{ w.p. } 0.98
    \end{array}
    \right. \text{ for } i\in[20], j\neq i
\end{equation}

\textbf{Random Linear Gaussian Graphical Models} refers to Gaussian Graphical Models whose edges that vary with time are sine functions of $t$ and the edges that depend on $t$ are randomly chosen with a Bernoulli distribution with probability 0.023. We set the diagonal element $1$ and off-diagonal elements are
\begin{equation}
    \Theta_{i,j}^\prime(t) = \Theta_{j,i}^\prime(t) = \left\{
    \begin{array}{l}
    0.45t \text{ w.p. } 0.023 \\
    0 \text{ w.p. } 0.977
    \end{array}
    \right. \text{ for } i\in[40], j\neq i
\end{equation}

\subsubsection{Construction of $\Theta_{0}$}
To ensure the positive definiteness and symmetry of the matrix $\Theta_{0}$, our procedure for constructing $\Theta_{0}$ is as follows:
First, we sample a matrix $A \in \mathbb{R}^{k \times k}$, where each element $A_{i,j} \sim \mathcal{N}(0,1)$.
We then compute $A^\top A / d /2$.
Finally, we replace the diagonal entries of $A^\top A$ with 0 and obtain $\Theta_0$. 

\subsection{Construction of Gaussian Graphical Models for Inference}\label{app:expset}
In \cref{sec:spartsmplus}, we considered two different types of Gaussian Graphical Models. 

\textbf{Deterministic Gaussian Graphical Models }refers to a Gaussian Graphical Model whose edges vary linearly with time and those edges are manually chosen. In our experimental setting, we set the diagonal elements above and below the main diagonal of $\Theta^\prime(t)$ to vary with time linearly, as well as the edge between the first node and third and forth node, specifically
\begin{equation}
    \Theta^\prime_{i,i+1}(t) = \Theta^\prime_{i+1,i}(t) = t \text{ for } 1\leq i\leq19 \text{ and } \Theta^\prime_{1,i}(t) = \Theta^\prime_{i,1}(t) = t \text{ for } i = 3,4,5
\end{equation}
and remaining elements are set at 0.\\
\textbf{Random Gaussian Graphical Models }refers to a Gaussian Graphical Model whose edges vary with time linearly are random. In our experiment setting, we set the off-diagonal elements except edge of interest to follow a Bernoulli distribution with probability $0.2$, i.e.
\begin{equation}
    \Theta_{i,j}^\prime(t) = \Theta_{j,i}^\prime(t)  = \left\{
    \begin{array}{l}
    t \text{ w.p. } 0.2 \\
    0 \text{ w.p. } 0.8
    \end{array}
    \right. \text{ for } i\in[20], j\neq i
\end{equation}
We then refill the diagonal entries to 0.

\subsubsection{Construction of $\Theta_{0}$}
To ensure the positive definiteness and symmetry of the matrix $\Theta_{0}$, our procedure for constructing $\Theta_{0}$ is as follows:
First, we sample a matrix $A \in \mathbb{R}^{k \times k}$, where each element $A_{i,j} \sim U(0,1)$ are sampled uniformly.
We then compute $\delta A^\top A$, where we set $\delta = 0.01$.
Finally, we replace the diagonal entries of $\delta A^\top A$ with 12 and obtain $\Theta_0$, where the choice of 12 ensures the positive definiteness of ${\Theta(t)}^{-1}$ in the power test experiments.

\subsection{Weighting Function}
\label{sec:weighting.function}
In Section \ref{sec:estimating:obj}, we specify the weight function $g(t)$ with the condition that it equals zero at the boundaries of the time domain. 
For truncated score matching, \citet{song} propose a distance function as $g$, from which we take inspiration.
Let $t_{\text{start}}$ and $t_{\text{end}}$ denote the start and end of the time domain respectively. We propose
\begin{equation}
g(t) \coloneqq -(t - t_{\text{start}})(t - t_{\text{end}}),
    \label{eq:g}
\end{equation}
and consequently $\partial_t g_t = -(2t - t_{\text{start}} - t_{\text{end}})$. 
In experimental results, we have observed that the choice of $g$ does not have significant impact on the performance.

\subsection{Hyperparameter Choice}
In the coverage experiments, we set the regularization parameter for Lasso to be $\lambda_{\mathrm{lasso}} = \sqrt{2\frac{\log k}{n}}$, and we use the same value for the regularization parameter in the inverse Hessian estimation.

\subsection{Loggle for Testing Changing Edge}
\label{loggle_detail}
Loggle\citep{yang2018estimating}, built as an R package, is the main method we compared with in both estimation and inference task. We generate data and call the R package of Loggle from python. 

The main challenge we face when implementing the Loggle in comparison is how to turn $\Theta(t)$ obtained by Loggle into $\partial_t \Theta(t)$ which is related to the change of edge. Here we use a heuristic approach, permutation tests to find appropriate quantiles for deciding whether the edge should be considered changing. We shuffle the data matrix 100 times so that the dependency between $\mb x$ and $t$ are broken. Then we apply Loggle to the shuffled data and use least square linear regression to obtain the slope. The $2.5\%$ and $97.5\%$ quantiles of the set of slopes are set as thresholds; any slope from original data that falls outside the $95\%$ coverage indicates a changing edge. When calculating the coverage and power, we compute new quantiles for each trial individually.

\end{document}



% Supplementary material: To improve readability, you must use a single-column format for the supplementary material.
\onecolumn
\aistatstitle{Instructions for Paper Submissions to AISTATS 2024: \\
Supplementary Materials}

\section{FORMATTING INSTRUCTIONS}

To prepare a supplementary pdf file, we ask the authors to use \texttt{aistats2024.sty} as a style file and to follow the same formatting instructions as in the main paper.
The only difference is that the supplementary material must be in a \emph{single-column} format.
You can use \texttt{supplement.tex} in our starter pack as a starting point, or append the supplementary content to the main paper and split the final PDF into two separate files.

Note that reviewers are under no obligation to examine your supplementary material.

\section{MISSING PROOFS}

The supplementary materials may contain detailed proofs of the results that are missing in the main paper.

\subsection{Proof of Lemma 3}

\textit{In this section, we present the detailed proof of Lemma 3 and then [ ... ]}

\section{ADDITIONAL EXPERIMENTS}

If you have additional experimental results, you may include them in the supplementary materials.

\subsection{The Effect of Regularization Parameter}

\textit{Our algorithm depends on the regularization parameter $\lambda$. Figure 1 below illustrates the effect of this parameter on the performance of our algorithm. As we can see, [ ... ]}

\vfill

%% file: notation_song.tex
\makeatletter
\DeclareFontFamily{U}{tipa}{}
\DeclareFontShape{U}{tipa}{m}{n}{<->tipa10}{}
\newcommand{\arc@char}{{\usefont{U}{tipa}{m}{n}\symbol{62}}}%

\newcommand{\arc}[1]{\mathpalette\arc@arc{#1}}

\newcommand{\arc@arc}[2]{%
  \sbox0{$\m@th#1#2$}%
  \vbox{
    \hbox{\resizebox{\wd0}{\height}{\arc@char}}
    \nointerlineskip
    \box0
  }%
}
\makeatother

\newcommand{\argmin}{\mathop{\rm argmin}\limits}

\newcommand{\boldtheta}{{\boldsymbol{\theta}}}

\newcommand{\boldalpha}{{\boldsymbol{\alpha}}}

\newcommand{\boldTheta}{{\boldsymbol{\Theta}}}
\newcommand{\boldf}{{\boldsymbol{f}}}

\newcommand{\boldx}{{\boldsymbol{x}}}

\newcommand{\boldF}{{\boldsymbol{F}}}

\newcommand{\boldy}{{\boldsymbol{y}}}

\newcommand{\mathbbE}{\mathbb{E}}

\newcommand{\dx}{\mathrm{d}\boldx}

\newcommand{\vertiii}[1]{{\left\vert\kern-0.25ex\left\vert\kern-0.25ex\left\vert #1 
    \right\vert\kern-0.25ex\right\vert\kern-0.25ex\right\vert}}

%% file: notation_danny.tex
\newcommand{\mb}[1]{\boldsymbol{#1}}

\newcommand{\br}[1]{\left( #1 \right)}
\newcommand{\brs}[1]{\left[ #1 \right]}
\newcommand{\brangle}[1]{\left\langle #1 \right\rangle}

\newcommand{\E}{\mathbb{E}}
\newcommand{\R}{\mathbb{R}}

\newcommand\myeq[1]{\stackrel{\mathclap{\tiny\mbox{#1}}}{=}}

%% file: notation_leyang.tex
\newcommand{\ba}{\boldsymbol{\alpha}}

\newcommand{\bome}{\boldsymbol{\omega}}

%% file: notation_overwrites.tex
\newcommand{\p}{p(\mb{x}; \mb\theta)}   % p(x;theta)
\newcommand{\psmall}{p_{\mb\theta}}
\newcommand{\pbar}{\bar{p}(\mb{x}; \mb\theta)}   % \bar{p}(x;theta)
\newcommand{\partialthetasmall}{\partial_t \mb{\theta}(t)}

\newcommand{\partialdthetasmall}{\partial_t^2 \mb{\theta}_t}

\newcommand{\score}[1]{{\nabla_{\mb{x}}}\log{#1}(\boldx)}

\newcommand{\scoret}[1]{\partial_t \log {#1}(\boldx)}

\newcommand{\Eqtsmall}{\E_{q_t}}

\newcommand{\model}{s(\mb{x}, t)}
\newcommand{\modelsmall}{s_{t}}
\newcommand{\modelsmallsample}[1]{s(#1)}

\newcommand{\timeintegral}{\int_{t=0}^{t=1}}

\newcommand{\xintegral}{\int_{\mathcal{X}}}